\documentclass[journal, 10pt]{IEEEtran}
\usepackage{geometry}

\usepackage[utf8]{inputenc} % allow utf-8 input
\usepackage[T1]{fontenc}    % use 8-bit T1 fonts
\usepackage{hyperref}       % hyperlinks
\usepackage{url}            % simple URL typesetting
\usepackage{booktabs}       % professional-quality tables
\usepackage{amsfonts}       % blackboard math symbols
\usepackage{nicefrac}       % compact symbols for 1/2, etc.
\usepackage{microtype}      % microtypography
\usepackage{lipsum}
\usepackage{fancyhdr} % header
\usepackage{graphicx}  
\usepackage{subcaption}  % For subfigure environment% graphics
\usepackage[cmex10]{amsmath} 
\usepackage[T1]{fontenc}
\usepackage{ifthen}
\usepackage{cite}
\usepackage{color}
\usepackage[dvipsnames]{xcolor}
\usepackage{pifont} % For checkmark and cross symbols
\usepackage{latexsym}
\usepackage{amssymb,mathtools,amsthm}
\usepackage{amsfonts}
\usepackage{dsfont}
\usepackage{amsopn}
\usepackage{array}
\def\epsilon{\varepsilon}
\def\eps{\varepsilon}
\usepackage{algorithm}
\usepackage{algpseudocode}
\usepackage{graphicx}

\DeclareMathOperator*{\argmin}{arg\,min}  % Defining argmin
\DeclareMathOperator*{\argmax}{arg\,max}
\newtheorem{theorem}{Theorem}
\newtheorem{lemma}[theorem]{Lemma}

\newtheorem{definition}{Definition}

\allowdisplaybreaks

\begin{document}
\onecolumn
%% Title
\title{Optimal Fairness under Local Differential Privacy
%%%% Cite as
%%%% Update your official citation here when published 
% \thanks{\textit{\underline{Citation}}: 
% \textbf{Authors. Title. Pages.... DOI:000000/11111.}} 
% 
}
\author{
  Hrad Ghoukasian and Shahab Asoodeh \\
  Department of Computing and Software, McMaster University \\
  \texttt{ \{ghoukash, asoodeh\}@mcmaster.ca} \\
  %% examples of more authors
  %% \AND
  %% Coauthor \\
  %% Affiliation \\
  %% Address \\
  %% \texttt{email} \\
  %% \And
  %% Coauthor \\
  %% Affiliation \\
  %% Address \\
  %% \texttt{email} \\
  %% \And
  %% Coauthor \\
  %% Affiliation \\
  %% Address \\
  %% \texttt{email} \\
}

\maketitle

% \SA{Better framework: Synthetic data generation with fairness and privacy guarantee.  related works: 
% \begin{itemize}
%     \item \url{https://proceedings.mlr.press/v238/abroshan24a/abroshan24a.pdf}
%     \item \url{https://arxiv.org/pdf/2307.00161}
%     \item \url{https://openreview.net/pdf?id=tge5NiX4CZo}
%     \item \url{https://scholarworks.sjsu.edu/cgi/viewcontent.cgi?article=9062&context=etd_theses}
%     \item \url{https://arxiv.org/pdf/2110.12884v2}
% \end{itemize}
% We need to first compare our work with these works in whatever way their experiments are done, and then compare in terms of classification error with fair-projection.  \\
% ~\\
% Long-shot: Now that you're familiar with the mollifier concept, we can aim to use mollification (or sampling) to generate fair dataset. See this: \url{https://proceedings.mlr.press/v202/soen23a/soen23a.pdf}.
% }

% Abstract
\begin{abstract}
We investigate how to optimally design local differential privacy (LDP) mechanisms that reduce data unfairness and thereby improve fairness in downstream classification. We first derive a closed-form optimal mechanism for binary sensitive attributes and then develop a tractable optimization framework that yields the corresponding optimal mechanism for multi-valued attributes. As a theoretical contribution, we establish that for discrimination–accuracy optimal classifiers, reducing data unfairness necessarily leads to lower classification unfairness, thus providing a direct link between privacy-aware pre-processing and classification fairness.

Empirically, we demonstrate that our approach consistently outperforms existing LDP mechanisms in reducing data unfairness across diverse datasets and fairness metrics, while maintaining accuracy close to that of non-private models. Moreover, compared with leading pre-processing and post-processing fairness methods, our mechanism achieves a more favorable accuracy–fairness trade-off while simultaneously preserving the privacy of sensitive attributes. Taken together, these results highlight LDP as a principled and effective pre-processing fairness intervention technique.

%%%%% OLD%%%%%
% In this work, we design optimal pre-processing mechanisms based on local differential privacy (LDP) to reduce data unfairness and, in turn, improve classification fairness. For binary sensitive attributes, we derive a closed-form optimal mechanism. For multi-valued attributes, we cast the problem as a numerical optimization that yields the optimal mechanism. We further prove, using the notion of discrimination-accuracy optimal classifiers, that reducing data unfairness provably reduces downstream classification unfairness.
% Empirically, we evaluate our approach across multiple datasets and fairness metrics, showing that it outperforms existing LDP mechanisms in reducing unfairness while maintaining accuracy comparable to non-private models. Moreover, compared to state-of-the-art fairness pre-processing and post-processing methods, our mechanism achieves a better accuracy–fairness trade-off while preserving the privacy of sensitive attributes—highlighting LDP as a powerful fairness-aware pre-processing tool.

\end{abstract}

% Introduction
\section{Introduction}\label{sec: introduction}

Machine learning algorithms are increasingly used in high-stakes decision-making tasks, making it essential to ensure fairness in their outcomes. Fairness in machine learning aims to prevent models from discriminating against demographic groups characterized by \textit{sensitive} attributes such as gender or race. In practice, fairness is context-dependent, and no single definition applies universally. Metrics such as statistical parity \cite{feldman2015certifying} and equalized opportunity \cite{hardt2016equality} formalize different ways of assessing whether model decisions remain independent of sensitive attributes.
Beyond fairness, privacy is another core component of trustworthy machine learning, as it protects individuals’ data in training datasets. Differential privacy (DP) \cite{dwork2006calibrating,dwork2006differential}  is the de facto standard for providing formal privacy guarantees in
machine learning. DP is studied in both the central model, where a trusted curator is assumed, and the local DP (LDP) model, which assumes no trusted entity. Prior work shows that central DP can worsen fairness and may even be fundamentally incompatible with it \cite{bagdasaryan2019differential,ganev2022robin,pujol2020fair,farrand2020neither,cummings2019compatibility}. In contrast, the interaction between fairness and LDP remains largely unexplored, and no analogous incompatibility result is known.

Widely deployed systems using LDP algorithms have been implemented by companies such as Google, Apple, and Microsoft \cite{erlingsson2014rappor, Apple_Privacy, ding2017collecting}. Although some work has examined the intersection of fairness and LDP, key challenges remain. For instance, \cite{mozannar2020fair} proposes a learning scheme for training non-discriminatory classifiers with theoretical guarantees when only a privatized version of the sensitive attribute is available. Their approach employs an $\epsilon$-LDP mechanism—generalized randomized response (GRR) \cite{kairouz2014extremal}. Similarly, \cite{chen2022fairness} considers a semi-private setting where a small subset of users reveal their sensitive attributes, while the rest apply an $\epsilon$-LDP protocol. While these frameworks provide approaches for fair prediction with LDP-privatized data, they either focus exclusively on binary sensitive attributes or limit themselves to using GRR as the LDP mechanism, highlighting the need for further exploration.

Recently, several works have begun examining the impact of LDP on fairness in classification \cite{arcolezi2025group,makhlouf2024systematic,makhlouf2024impact}. For instance, \cite{arcolezi2025group} empirically analyzed how different LDP mechanisms affect fairness and utility in binary classification, showing that privatizing sensitive attributes can improve fairness with minimal utility loss compared to training on non-private data. \cite{makhlouf2024impact} showed that stronger privacy enhances fairness and that privatizing multiple sensitive attributes reduces unfairness more effectively than privatizing one. Complementing these empirical results, \cite{makhlouf2024systematic} provided a theoretical analysis of how randomized response (RR) \cite{warner1965randomized}, the binary analogue of GRR, affects fairness under different privacy levels and data distributions. Their analysis identifies conditions under which privacy can either mitigate or exacerbate unfairness. Moreover, \cite{xie2024privacy} showed that applying LDP during pre-processing limits sensitive information leakage and improves fairness. 

Collectively, these findings contrast with results in the central DP setting, where stronger privacy is often linked to reduced fairness. Instead, they suggest that LDP may serve as a promising avenue for fairness amplification. This motivates further study of how to optimally perturb sensitive attributes—beyond RR and GRR—to improve fairness for both binary and non-binary attributes. Overall, applying an LDP mechanism can be viewed as a pre-processing strategy for enhancing fairness while preserving utility.

% In this work, we explore LDP as a pre-processing method aimed at enhancing fairness in classification tasks.

Fairness intervention methods are applied at different stages to mitigate bias: pre-processing before training \cite{calders2009building, wang2019repairing, kamiran2012data, celis2020data, calmon2017optimized, hajian2012methodology,chakraborty2021bias,peng2022fairmask,gohar2023towards,madras2018learning,zemel2013learning,soen2023fair,yu2024fairbalance}, in-processing during model training \cite{lowy2021stochastic,cho2020fairdensity,cho2020fairmutual,jiang2020wasserstein,mary2019fairness,prost2019toward,zhang2018mitigating,agarwal2018reductions}, and post-processing after generating predictions \cite{wei2020optimized,wei2021optimized,chzhen2019leveraging,pleiss2017fairness,kim2020fact,jiang2020identifying,yang2020fairness,alghamdi2022beyond}. Among these, pre-processing methods offer the most flexibility within the data science pipeline, as they operate independently of the modeling algorithm \cite{calmon2017optimized}. Our approach aligns closely with the pre-processing methods. 
% employing LDP as a data pre-processing technique aimed at enhancing fairness.
% In the context of fairness pre-processing methods, many approaches are limited to binary sensitive attributes \cite{calders2009building, wang2019repairing, kamiran2012data}, while others are restricted to discrete or categorical features, making them unsuitable for datasets with continuous variables \cite{celis2020data, calmon2017optimized}.

Building on recent evidence that LDP can promote fairness \cite{arcolezi2025group, makhlouf2024systematic, makhlouf2024impact} and motivated by the flexibility of pre-processing methods, we study how to design the optimal LDP mechanism to reduce unfairness before model training. Specifically, we formulate the problem of identifying an LDP-based pre-processing mechanism that minimizes data unfairness. We further provide a theoretical analysis showing how reducing data unfairness translates to improved classification fairness. Finally, we validate the effectiveness of the optimal mechanisms through extensive experiments across diverse datasets and fairness metrics.

More precisely, our contributions are as follows:
\begin{itemize}
 
  \item We introduce the problem of designing LDP mechanisms that minimize data unfairness for both binary and non-binary sensitive attributes while preserving non-trivial utility. For the binary case, we derive a closed-form optimal solution, and for the non-binary case, we reformulate the problem as a min–max linear fractional program solvable via the branch-and-bound method~\cite{jiao2022solving}.

\item We theoretically show that for a class of discrimination–accuracy optimal classifiers, training on less discriminatory data yields lower post-classification unfairness. This result provides theoretical justification for our objective of minimizing data unfairness.

\item Through extensive experiments on multiple datasets and fairness metrics, we validate the effectiveness of our optimal mechanisms in reducing unfairness while maintaining utility comparable to state-of-the-art LDP methods. Moreover, our mechanisms achieve a better accuracy–fairness trade-off than both pre- and post-processing fairness interventions, while preserving the privacy of sensitive attributes.

\end{itemize}

The remainder of this paper is structured as follows. Section \ref{sec: related_works} reviews the related work. In Section \ref{sec: notation and definition}, we introduce the key notation and definitions used throughout the paper, along with an overview of common LDP mechanisms. Section \ref{sec: technical_results} defines the problem we address and presents the optimal LDP mechanism for mitigating data unfairness, along with the corresponding theoretical results. Section \ref{sec: experiments} describes the experimental setup and validates the theoretical findings. Section \ref{sec: conclusion} wraps up the paper with a discussion of the key contributions and findings.

% Related Work
\section{Related Work}\label{sec: related_works}

We categorize the related work into three main groups. The first group includes studies that explore the intersection of LDP and fairness. The second focuses on studies that explore how central DP impacts fairness and vice versa. The third group contains papers that focus on pre-processing methods with the goal of improving fairness.

\textbf{LDP and Fairness:}
Several studies have examined the impact of LDP on fairness in classification. \cite{arcolezi2025group} empirically shows that applying LDP to sensitive attributes improves fairness with minimal utility loss compared to non-private data, highlighting that GRR and subset selection (SS) \cite{wang2016mutual} offer the best accuracy-fairness-privacy trade-offs among state-of-the-art LDP mechanisms. Expanding on this, \cite{makhlouf2024impact} experimentally demonstrates that stronger privacy (i.e., lower $\epsilon$) further enhances fairness, with greater reductions in disparity when LDP is applied to multiple sensitive attributes. They subsequently delineate this observation by characterizing conditions under which RR implies lower unfairness for binary sensitive attributes \cite{makhlouf2024systematic}. Other works focus on learning frameworks with privatized sensitive attributes. For instance, \cite{mozannar2020fair} adapts non-discriminatory classifiers to work with privatized attributes using GRR, offering theoretical performance guarantees on utility and fairness. Similarly, \cite{chen2022fairness} considers the scenario of semi-private sensitive attributes. Moreover, \cite{xie2024privacy} shows that incorporating LDP randomizers
during encoding can enhance fairness in representation learning while preserving utility. The work most closely related to ours is \cite{makhlouf2024systematic}, as they also theoretically analyze the impact of LDP on fairness. However, their study is limited to binary sensitive attributes and a specific LDP mechanism (RR) only. 

\textbf{Central DP and Fairness:}
\cite{bagdasaryan2019differential} empirically shows that central differential privacy, specifically differentially private stochastic gradient descent (DP-SGD) \cite{abadi2016deep}, has a disparate impact on accuracy of different subgroups. Similarly, \cite{ganev2022robin} demonstrates that applying DP to synthetic data generation disproportionately affects minority sub-populations. \cite{pujol2020fair} also finds that applying DP can exacerbate unfairness in decision-making tasks. \cite{farrand2020neither} further shows that even with loose privacy guarantees and minimal data imbalance, DP-SGD can lead to disparate impact. \cite{agarwal_2020} and \cite{cummings2019compatibility} introduce an incompatibility result, highlighting that differential privacy and fairness are at odds when aiming for non-trivial accuracy in learning algorithms. Recently, \cite{ko2024fairness} examined the impact of central DP on fairness beyond classification tasks. They demonstrated that using DP to set sampling rates for collecting socio-demographic data reduces unfairness across different population segments in resource allocation.
Additionally, some works explore the intersection of privacy and fairness beyond the scope of central DP or group fairness. For instance, \cite{chang2021privacy} examines the privacy risks of achieving group fairness, showing that fairness may compromise privacy, particularly through membership inference attacks, even when differential privacy is not explicitly applied. \cite{dwork2012fairness} provides a theoretical link between individual fairness and DP, arguing that individual fairness can be seen as a generalization of DP and outlining conditions under which a DP mechanism ensures individual fairness. 

\textbf{Pre-processing Techniques for Fairness:}
Pre-processing techniques for bias mitigation in fairness literature involve making changes to training data. Some methods modify values within the training data, such as altering ground truth labels (relabeling) \cite{calders2009building,hajian2012methodology,kamiran2012data} or adjusting other features (perturbation) \cite{wang2019repairing}. Methods such as \cite{kamiran2012data,yu2024fairbalance} mitigate bias through reweighting of existing samples, while \cite{wang2019repairing} perturbs the input distribution for disadvantaged groups to create a counterfactual distribution, particularly targeting binary sensitive attributes.
Sampling methods adjust training data by changing sample distributions (e.g., adding or removing samples) or adapting their influence on training \cite{chakraborty2021bias,celis2020data}. \cite{celis2020data} proposes an optimization-based framework to learn distributions over the data domain that stay close to the empirical distribution. This technique is applicable to datasets with discrete or categorical attributes, both sensitive and non-sensitive. Other methods augment training data with additional, ideally unbiased features \cite{madras2018learning}. Finally, some techniques learn transformations of the training data to reduce bias while retaining as much information as possible \cite{zemel2013learning,calmon2017optimized}. \cite{calmon2017optimized} introduces an optimization algorithm that modifies non-sensitive features and labels while keeping sensitive attributes unchanged, focusing on datasets with categorical or discrete attributes. Our method aligns with perturbation-based approaches, as we perturb the sensitive attribute using LDP.

% \cite{celis2020data} presents an optimization-based framework for learning distributions over the data domain that stay close to the empirical distribution. This pre-processing approach can be applied to feature sets with discrete or categorical attributes, including both sensitive and non-sensitive attributes.
% \cite{calmon2017optimized} introduces an optimization-based algorithm for pre-processing data that modifies non-sensitive features and labels while keeping the sensitive attributes unchanged. In this work, the features are assumed to be categorical or discrete. Pre-processing methods for fairness, as discussed in \cite{ kamiran2012data, king2001logistic}, aim to modify the dataset by re-weighting existing data points.
% \cite{calders2009building} also provides reweighting and relabeling methods for pre-processing the data for binary sensitive attributes.
% \cite{wang2019repairing} perturbs the distribution of input variables for the disadvantaged group, referring to the modified distribution as a counterfactual distribution. The method focuses on cases with a binary sensitive attribute.

% Main Results
\section{Notation and Definitions}\label{sec: notation and definition}

\subsection{Notation}\label{subsec: notation}
We consider a binary classification setting with joint distribution \( P_{XAY} \) over the triplet \( T = (X, A, Y) \), where \( X \in \mathcal{X} \subset \mathbb{R}^d \) denotes the non-sensitive features, \( A \) the sensitive attribute, and \( Y \in \{0,1\} \) the target label. The sensitive attribute may be binary (\( A \in \{0,1\} \)) or non-binary (\( A \in \{1,2,\ldots,k\} \), \( k \ge 2 \)). We introduce \( Z \), a perturbed version of \( A \) obtained via an LDP mechanism, such that the triplet \( (X, A, Y) \) becomes \( (X, Z, Y) \) after perturbation.  
For the distribution \( P_{XAY} \), or a dataset \( D \) consisting of i.i.d. samples drawn from it, let \( p_i = \Pr(A = i) \) and \( p_{1|i} = \Pr(Y = 1 \mid A = i) \). We denote by \( P_Y \) the marginal distribution of \( Y \) under the joint distribution \( P_{XAY} \).
In practice, since the true data distribution is unknown, we use \( D \) and \( P_{XAY} \) interchangeably, defining all metrics analogously for both. We denote the index set \( \{1, 2, \ldots, k\} \)  by \( [k] \). 

% and the symbol \( \left\lfloor \cdot \right\rceil \) is used to indicate rounding to the nearest positive integer.

\subsection{Fairness Metrics}\label{subsec: group fairness metrics}

In this section, we define two data unfairness metrics, \( \Delta \) and \( \Delta' \), as well as three classification fairness metrics: statistical parity, equalized opportunity, and (mean) equalized odds.

\begin{definition}[Data unfairness $\Delta$\cite{calmon2017optimized}]\label{def: data unfairness}
Let \( X \in \mathcal{X} \), \( A \in [k] \), and \( Y \in \{0,1\} \) denote non-sensitive features, the sensitive attribute, and the label with joint distribution \( P_{XAY} \). Each sample \((x_i, a_i, y_i)\) in dataset \( D \) is drawn i.i.d.\ from \( P_{XAY} \).
The data unfairness metric \(\Delta(D)\) associated with \(D\) or \(P_{XAY}\) is defined as:
\begin{equation*}
    \Delta(D) := \max_{a \in [k]} \left| \frac{\Pr(Y = 1 \mid A = a)}{\Pr(Y = 1)} - 1 \right|.
\end{equation*}
\end{definition}

\begin{definition}[Data unfairness $\Delta'$ \cite{kamiran2012data}]\label{def: data unfairness'}
Let \( X \in \mathcal{X} \), \( A \in [k] \), and \( Y \in \{0,1\} \) denote non-sensitive features, the sensitive attribute, and the label with joint distribution \( P_{XAY} \). Each sample \((x_i, a_i, y_i)\) in dataset \( D \) is drawn i.i.d.\ from \( P_{XAY} \). 
The data unfairness metric \(\Delta'(D)\) associated with \(D\) or \(P_{XAY}\) is defined as:
\begin{equation*}
    \Delta'(D) := \max_{a,a' \in [k]} \Bigl| \Pr(Y = 1 \mid A = a) - \Pr(Y = 1 \mid A = a')\Bigr|.
\end{equation*}
\end{definition}

Both data unfairness metrics, \( \Delta \) and \( \Delta' \), capture the dependence (or independence) of the actual labels on the sensitive attribute. It can be shown that \( \Delta(D) \leq c_1 \Delta'(D) \) and \( \Delta'(D) \leq c_2 \Delta(D) \) for some constants \( c_1 \) and \( c_2 \) that depend on the marginal distribution $P_Y$ of the joint distribution \( P_{XAY} \). (see Lemma \ref{lemma: data metrics are equivalent}). Therefore, the metrics \( \Delta \) and \( \Delta' \) are essentially the same. 
% However, for technical reasons that will become clear in the following sections, we use Definition \ref{def: data unfairness} to formulate the optimal LDP-based pre-processing problem in Sections \ref{subsec:optimal binary_A} and \ref{subsec: optimal non-binary_A}. In contrast, we use Definition \ref{def: data unfairness'} in Section \ref{subsec: data fairness leads to classification fairness} to establish the connection between data fairness and classification fairness. 
However, for technical reasons that will become clear in the following sections, we use Definition \ref{def: data unfairness'} to formulate the optimal LDP-based pre-processing problem for binary sensitive attributes in Section \ref{subsec:optimal binary_A} and to establish the relationship between data unfairness and classification unfairness in Section \ref{subsec: data fairness leads to classification fairness}. In Section \ref{subsec: optimal non-binary_A}, we utilize Definition \ref{def: data unfairness} to formulate the optimal LDP-based pre-processing problem for non-binary sensitive attributes.

\begin{definition}[Statistical parity \cite{feldman2015certifying}]\label{def: statistical par}  
Let \(X \in \mathcal{X}\), \(A \in [k]\), and \(Y \in \{0, 1\}\) be random variables representing non-sensitive features, sensitive attributes, and labels, respectively, with a joint distribution \(P_{XAY}\). Let \(\hat{Y} = \hat{h}(X, A)\) be a binary classifier, where \(\hat{h} : \mathcal{X} \times [k] \to \{0, 1\}\). The statistical parity gap \(\Delta_{\mathrm{SP}}(\hat{h})\) associated with the classifier \(\hat{h}\) and distribution \(P_{XAY}\) is defined as:
% \begin{equation*}
%     \Delta_{\mathrm{SP}}(\hat{h}) = \max_{a \in [k]} \left| \frac{\Pr(\hat{Y} = 1 \mid A = a)}{\Pr(\hat{Y} = 1)} - 1 \right|.
% \end{equation*}
\begin{equation*}
    \Delta_{\mathrm{SP}}(\hat{h}) := \max_{a,a' \in [k]} \left| \Pr(\hat{Y} = 1 \mid A = a) - \Pr(\hat{Y} = 1 \mid A = a') \right|.
\end{equation*}

% We say that \(\hat{Y} = \hat{h}(X, A)\) satisfies \(\gamma\)-statistical parity if \(\Delta_{\mathrm{SP}}(\hat{h}) \leq \gamma\).
\end{definition}

\begin{definition}[Equalized opportunity \cite{hardt2016equality}]
Let \(X \in \mathcal{X}\), \(A \in [k]\), and \(Y \in \{0, 1\}\) be random variables representing non-sensitive features, sensitive attributes, and labels, respectively, with a joint distribution \(P_{XAY}\). Let \(\hat{Y} = \hat{h}(X, A)\) be a binary classifier, where \(\hat{h} : \mathcal{X} \times [k] \to \{0, 1\}\). The equalized opportunity gap \(\Delta_{\mathrm{EO}}(\hat{h})\) associated with the classifier \(\hat{h}\) and distribution \(P_{XAY}\) is defined as:
\begin{equation*}
    \Delta_{\mathrm{EO}}(\hat{h}) := \max_{a,a'\in[k]}
   \bigl|
   \mathrm{TPR}_{A=a}
   - \mathrm{TPR}_{A=a'}
   \bigr|,
\end{equation*}
where \(\mathrm{TPR}_{A=a} = \Pr(\hat{Y} = 1 \mid Y = 1, A = a)\).
\end{definition}

\begin{definition}[(Mean) equalized odds \cite{hardt2016equality,alghamdi2022beyond}]\label{def: mean equalized odds}
Let \(X \in \mathcal{X}\), \(A \in [k]\), and \(Y \in \{0, 1\}\) be random variables representing non-sensitive features, sensitive attributes, and labels, respectively, with a joint distribution \(P_{XAY}\). Let \(\hat{Y} = \hat{h}(X, A)\) be a binary classifier, where \(\hat{h} : \mathcal{X} \times [k] \to \{0, 1\}\). The equalized odds gap  \(\Delta_{\mathrm{EOd}}(\hat{h})\) and mean equalized odds gap \(\Delta_{\mathrm{MEO}}(\hat{h})\) associated with the classifier \(\hat{h}\) and distribution \(P_{XAY}\) are defined as:
\begin{align*}
    \Delta_{\mathrm{MEO}}(\hat{h}) & :=  \max_{a, a' \in [k]} \frac{1}{2} \Bigl( \bigl| \mathrm{TPR}_{A=a} - \mathrm{TPR}_{A=a'} \bigr| + \bigl| \mathrm{FPR}_{A=a} - \mathrm{FPR}_{A=a'} \bigr| \Bigr),\\
    \Delta_{\mathrm{EOd}}(\hat{h}) & := 
    \max_{a, a' \in [k]}
    \max \Bigl\{
        \bigl| \mathrm{TPR}_{A=a} - \mathrm{TPR}_{A=a'} \bigr|,
        \bigl| \mathrm{FPR}_{A=a} - \mathrm{FPR}_{A=a'} \bigr|
    \Bigr\}.
\end{align*}
where \(\mathrm{TPR}_{A=a} = \Pr(\hat{Y} = 1 \mid Y = 1, A = a)\) and \(\mathrm{FPR}_{A=a} = \Pr(\hat{Y} = 1 \mid ~Y = 0,~A = a)\).
\end{definition}

It is worth noting that definitions of classification unfairness extend beyond the three mentioned here. Numerous definitions exist that may be useful to consider depending on the specific context of the problem \cite{chouldechova2017fair, dwork2012fairness, berk2021fairness, corbett2017algorithmic, kilbertus2017avoiding,kleinberg2016inherent,kusner2017counterfactual,sabato2024fairnessunfairnessbinarymulticlass}.

\subsection{Discrimination-Accuracy Optimality}\label{subsec: DA optimal}
\begin{definition}[Discrimination-accuracy optimal classifier \cite{kamiran2012data}] \label{def: DA optimal}
For classifiers \(h\) and \(h'\), we say  \(h\) dominates \(h'\) 
if it achieves at least the same accuracy and no higher discrimination 
(measured by an unfairness metric such as equalized opportunity or statistical parity). 
Classifier \(h\) strictly dominates \(h'\) if at least one inequality is strict. 
A classifier \(h \in \mathcal{H}\) is said to be discrimination–accuracy optimal 
(DA-optimal) in \(\mathcal{H}\) if no other classifier in \(\mathcal{H}\) strictly dominates it.

\end{definition}

\subsection{Local Differential Privacy}\label{subsec: LDP definition}
\begin{definition}[Local differential privacy \cite{warner1965randomized,evfimievski2003limiting}] A randomized algorithm \(\mathcal{M}: \mathcal{D} \to \mathcal{R}\) is said to satisfy \(\epsilon\)-local differential privacy (\(\epsilon\)-LDP), where \(\epsilon > 0\), if for any pair of input values \(x_1, x_2 \in \mathcal{D}\), and for any possible output \(y \in  \mathcal{R}\), it holds that
\begin{equation*}
   \Pr\big(\mathcal{M}(x_1) = y\big) \leq e^{\epsilon} \Pr\big(\mathcal{M}(x_2) = y\big).
\end{equation*}

\end{definition}

\subsection{Common LDP Mechanisms}\label{subsec: Common LDP mech}

\textbf{Generalized randomized response (GRR)}\cite{warner1965randomized, kairouz2014extremal}: Given a sensitive attribute \(A\) taking values in \([k]\), the generalized randomized response (GRR) mechanism perturbs the true value \(a \in [k]\) to preserve privacy. GRR outputs the true value \(a\) with probability \(\pi\), and any other possible value \(a' \in [k] \setminus \{a\}\) with a different probability \(\bar{\pi}\). Specifically, for the perturbed output \(Z\):
\begin{equation*}
    \forall z \in [k] :\hspace{0.5cm} \Pr\big(Z = z \mid A = a\big) = 
\begin{cases} 
\pi = \frac{e^\epsilon}{e^\epsilon + k - 1} & \text{if } z = a,\\
\Bar{\pi} = \frac{1}{e^\epsilon + k - 1} & \text{if } z \neq a,
\end{cases}
\end{equation*}
In binary $A$ case, the mechanism is known as randomized response (RR).

% where \(Z\) is the perturbed value sent to the server, \(\epsilon > 0\) is the privacy parameter, and \(p\) and \(q\) are the probabilities of retaining or perturbing the value of \(A\), respectively.

% In this mechanism:
% - The parameter \(\epsilon\) controls the level of privacy, with larger \(\epsilon\) indicating less noise and higher data utility.
% - The probability \(p\) is the likelihood of the true value being retained, while \(q\) is the likelihood of any other value being selected.

% This formulation allows for protection of sensitive information by making the reported value \(Z\) less directly linked to the true value \(A\).

\textbf{Subset selection (SS) \cite{wang2016mutual}:}
Given a sensitive attribute \(A\) taking values in \([k]\), SS mechanism perturbs the true value \(a \in [k]\) by reporting a subset of values \(\Omega \subseteq [k]\). The mechanism aims to include the true value \(a\) in the subset \(\Omega\) with a higher probability than any other value in \([k] \setminus \{a\}\). It is proved that the optimal subset size is \(\omega=\max\Big( 1, \bigl\lfloor \tfrac{k}{e^\epsilon+1} \bigr\rceil\Big)\) \cite{ye2018optimal}. The SS mechanism proceeds as follows:
\begin{enumerate}
    \item Initialize an empty subset \(\Omega\).
    \item Add the true value \(a\) to \(\Omega\) with probability
    \[
    p = \frac{\omega e^\epsilon}{\omega e^\epsilon + k - \omega}.
    \]
    \item Fill \(\Omega\) as follows:
    \begin{itemize}
        \item If \(a \in \Omega\), sample \(\omega - 1\) additional values uniformly at random (without replacement) from \([k] \setminus \{a\}\) and add them to \(\Omega\).
        \item If \(a \notin \Omega\), sample \(\omega\) values uniformly at random (without replacement) from \([k] \setminus \{a\}\) and add them to \(\Omega\).
    \end{itemize}
\end{enumerate}
Finally, the user sends the subset \(\Omega\) to the server as the perturbed value. As we can see, when $\omega = 1$, SS is equivalent to GRR.

% Main Results
\section{Main Results}\label{sec: technical_results}

As highlighted in Introduction, \cite{arcolezi2023local} empirically demonstrates that applying different LDP mechanisms to the sensitive attributes of a dataset, followed by classifier training, can enhance classification fairness without significantly sacrificing model utility. This improvement is observed when compared to a non-private scenario, where the classifier is trained directly on the original non-private data in binary classification tasks. In our setting, sensitive attributes such as race and gender are privacy-sensitive. Therefore, we ensure privacy by applying LDP mechanisms to sensitive attributes, following the approach of \cite{friedberg2023privacy}, which guarantees LDP for each user's group membership.

In this work, we aim to provide a theoretical foundation for this empirical observation. In Section~\ref{subsec: GRR and Data Fairness}, we begin by analyzing GRR as an LDP approach that offers a highly effective trade-off between fairness, utility, and privacy \cite{arcolezi2023local}. Our goal is to understand how GRR impacts data fairness. In Section~\ref{subsec:optimal binary_A}, we focus on the case of a binary sensitive attribute ($A \in \{0,1\}$) and identify the optimal LDP mechanism that can minimize unfairness in the dataset. Then, in Section~\ref{subsec: optimal non-binary_A}, we extend our analysis to non-binary sensitive attributes ($A \in [k], k > 2$). Here, we generalize the optimal LDP mechanism for fairness, which leads to solving a min-max linear fractional programming \cite{jiao2022solving}. Finally, in Section \ref{subsec: data fairness leads to classification fairness}, we provide theoretical conditions that establish the connection between data unfairness and classification unfairness.

\subsection{Data Unfairness under GRR}\label{subsec: GRR and Data Fairness}
Let \(X \in \mathcal{X}\), \(A \in [k]\), and \(Y \in \{0,1\}\) be random variables representing non-sensitive features, sensitive attributes, and labels, respectively, with a joint distribution \(P_{XAY}\). Each data point in the dataset \(D\) is a triplet \((x_i, a_i, y_i)\), where \((x_i, a_i, y_i) \sim P_{XAY}\).
Now, suppose we independently perturb the sensitive attribute \(a_i\) of each data point in \(D\) using GRR to generate a new dataset \(D^{\varepsilon}_{GRR}\), keeping the same \(x_i\) and \(y_i\). The resulting dataset is \(D^{\varepsilon}_{GRR} = \{(x_i, z_i, y_i)\}_{i=1}^{n}\), where \(z_i\) is the noisy version of \(a_i\) obtained via GRR.
We define the data unfairness \(\Delta'(D^{\varepsilon}_{GRR})\) in a similar manner to Definition~\ref{def: data unfairness'} as:
\begin{equation*}
    \Delta'(D^{\varepsilon}_{GRR}) := \max_{z,z' \in [k]} \Bigl| \Pr(Y = 1 \mid Z = z) - \Pr(Y = 1 \mid Z = z' )\Bigr|,
\end{equation*}
where \(Z\) is the random variable representing the perturbed sensitive attribute after applying GRR to \(A\). These probabilities depend on both the original joint distribution \(P_{XAY}\) and the randomness introduced by GRR.

In the following lemma, we analyze the impact of applying GRR to the sensitive attributes of a dataset on data unfairness.

\begin{lemma}\label{lemma: RR makes data fairer}
     For binary \(Y\) and non-binary \(A\), applying GRR to the sensitive attributes \(A\) of a dataset results in $\Delta'(D^{\varepsilon}_{GRR}) \leq \Delta'(D)$.
\end{lemma}
This implies that applying GRR to the sensitive attributes of a dataset can help reduce data unfairness. Motivated by this, we investigate the problem of finding an optimal LDP mechanism that minimizes data unfairness. 
For clarity, we begin with the binary sensitive attribute case. In Section~\ref{subsec:optimal binary_A}, we address this by formulating the optimal LDP mechanism for enhancing data fairness in the case of binary sensitive attributes.

\subsection{Optimal LDP Mechanism: Binary Sensitive Attributes} \label{subsec:optimal binary_A}

Consider a dataset \(D = \{(x_i, a_i, y_i)\}_{i=1}^{n}\). 
Each sensitive attribute \(a_i\) is independently perturbed by an \(\varepsilon\)-LDP mechanism \(M\), producing a new dataset 
\(D_{M} = \{(x_i, z_i, y_i)\}_{i=1}^{n}\), 
where \(z_i\) is the noisy version of \(a_i\) generated by the randomized mapping \(M\) characterized as follows:
\begin{align}\label{eqn: general LDP mechanism}
\Pr(Z = z \mid A = 0) &= 
\begin{cases}
p, & z = 0,\\
1-p, & z = 1,
\end{cases} \nonumber \\[1mm]
\Pr(Z = z \mid A = 1) &=
\begin{cases}
q, & z = 1,\\
1-q, & z = 0.
\end{cases}
\end{align}
Here, \(Z\) denotes the output of the mechanism \(M\) applied to \(A\). 
The parameters \(p\) and \(q\) represent the probabilities of correctly reporting 
\(A=0\) and \(A=1\), respectively. 
Note that RR is obtained as a special case with 
\(p = q = \tfrac{e^{\varepsilon}}{e^{\varepsilon} + 1}\).

For a given privacy budget \(\varepsilon\), our goal is to find the optimal 
\(\varepsilon\)-LDP mechanism \(M\) that minimizes the unfairness 
of the perturbed data. 
Specifically, we seek to minimize the ratio $
\frac{\Delta'(D_{M})}{\Delta'(D)},$
where \(\Delta'(D)\) depends only on the original data distribution 
\(P_{XAY}\) and is thus constant.

We define the privacy level of a mechanism \( M \) as 
\[
\varepsilon^\star(M) \coloneqq \inf \{ \eps' \ge 0 : M \text{ is } \eps'\text{-LDP} \}.
\]
Intuitively, \( \varepsilon^\star(M) \)  denotes the smallest privacy parameter $\eps'$
for which $M$ satisfies LDP.
For a given $\eps$, if we consider the objective function $
\min\limits_{\eps-\text{LDP} M} 
\frac{\Delta'(D_{M})}{\Delta'(D)},$
the optimal fairness mechanism would be a fully random mechanism with 
\( p = q = \tfrac{1}{2} \), which trivially satisfies 
\( \varepsilon \)-LDP for any \( \varepsilon \ge 0 \). 
However, this mechanism offers no utility. 
To achieve a more meaningful trade-off between privacy and fairness, 
we instead consider the refined objective function: 
\begin{equation}\label{eqn: objective function binary optimal mechansim}
  \min\limits_{M : \, \varepsilon^\star(M) = \varepsilon} 
  \frac{\Delta'(D_{M})}{\Delta'(D)}.
\end{equation}
This formulation considers mechanisms with privacy level~\( \varepsilon \), 
ensuring meaningful (non-trivial) utility. 
For a given \( \varepsilon \), we seek the optimal mechanism \( M \)—or equivalently, 
the optimal parameters \( p^\star \) and \( q^\star \)—that minimize the objective 
in \eqref{eqn: objective function binary optimal mechansim}. 
The following theorem characterizes this optimal \( \varepsilon \)-LDP mechanism 
for the binary sensitive attribute case.

\begin{theorem}\label{thm: binary optimal mechanism}
Consider the case of binary \( Y \) and binary~\( A \), where \( p_{1|0} \leq p_{1|1} \). Let \( (p^*, q^*) \) represent the optimal parameters that minimize the objective function defined in (\ref{eqn: objective function binary optimal mechansim}). The optimal LDP mechanism is determined as follows:
\begin{align*}
    \text{If} \hspace{0.5cm} p_0 < p_1, \quad (p^*, q^*) = \bigl(1 - \frac{e^{-\varepsilon}}{2}, \frac{1}{2} \bigr),\\
    \text{If} \hspace{0.5cm} p_1 < p_0, \quad (p^*, q^*) = \bigl( \frac{1}{2}, 1 - \frac{e^{-\varepsilon}}{2} \bigr),
\end{align*}
where \( p \) and \( q \) are the parameters of the general LDP mechanism as defined in (\ref{eqn: general LDP mechanism}). In the tie case \(p_0 = p_1\), both solutions yield the same optimal value.

\end{theorem}

This theorem identifies the optimal mechanism for applying LDP as a pre-processing strategy in the binary sensitive attribute case. 
% Essentially, LDP is used as a tool to pre-process data and minimize data unfairness. 
However, sensitive attributes are not always binary, leading to the natural inquiry of how to determine the optimal pre-processing mechanism under LDP constraints for non-binary sensitive attributes. In the next section, we explore this issue.

\subsection{Optimal LDP Mechanism:  Non-Binary Sensitive Attributes} \label{subsec: optimal non-binary_A}
As shown in Theorem~\ref{thm: binary optimal mechanism}, 
the optimal LDP pre-processing mechanism for a binary sensitive attribute 
admits a closed-form solution, where \((p^\star, q^\star)\) depend on the 
data distribution and privacy parameter \(\varepsilon\). 
We next extend this result to non-binary sensitive attributes.

We want to find the optimal LDP mechanism for the binary $Y$ and non-binary $A$ case. Let $A \in [k]$, and let $\mathbf{Q}$ be a $k \times k$ matrix containing the parameters of the LDP mechanism. The randomized mechanism changes the sensitive attribute of the original data from $A = a$ to $Z = z$ using the parameters of $\mathbf{Q}$, defined as:

\begin{equation*}
    \Pr(Z = j \mid A = i) = q_{ij} \hspace{2cm} \forall i,j \in [k],
\end{equation*}

where \( q_{ij} \) represents the element in the \( i \)-th row and \( j \)-th column of the matrix \( \mathbf{Q} \). Similar to Section \ref{subsec:optimal binary_A}, the perturbed dataset is \(D_{M} = \{(x_i, z_i, y_i)\}_{i=1}^{n}\). The idea of the perturbation matrix \( \mathbf{Q} \) follows \cite{celis2021fair}, where this matrix is used to generate noisy sensitive attributes. However, in their setting, the perturbation mechanism is not an LDP mechanism. 
 The matrix $\mathbf{Q}$ should satisfy the following constraints:

\begin{enumerate}
    \item \textbf{LDP Constraint:} The matrix $\mathbf{Q}$ should satisfy $\epsilon$-LDP, i.e.,
    \begin{equation}\label{Q_constraint_LDP}
        q_{ij} - e^\eps q_{i'j} \leq 0 \hspace{2cm} \forall i,i',j \in [k] \quad \text{with } i \neq i'.
    \end{equation}

    \item \textbf{Row-Stochastic Constraint:} The matrix $\mathbf{Q}$ should satisfy 
\begin{align}\label{Q_constraint_row-stochastic}
    \sum_{j=1}^k q_{ij} = 1 \;\; \forall i \in [k], \qquad q_{ij} \ge 0 \;\; \forall i,j \in [k].
\end{align}

    \item \textbf{Truthfulness Constraint:} The probability of truly representing the sensitive attribute should be larger than or equal to the probability of misrepresenting it, i.e.,
    \begin{align}\label{Q_constraints_truth_and_lie_1} & q_{ii} \geq q_{ij} \hspace{2cm} \forall i,j \in [k], \\ & q_{jj} \geq q_{ij} \hspace{2cm} \forall i,j \in [k]. \label{Q_constraints_truth_and_lie_2} \end{align}

    \item \textbf{Utility Constraint:} The mechanism should satisfy a utility metric, where the probability of error (i.e., total probability of $Z \neq A$) should be smaller than some predefined constant $\zeta$. This constraint is formulated as:
    \begin{align} 
       & \Pr(A = Z) \geq 1-\zeta. \nonumber
    \end{align}

    We know: 
    \begin{align*}
         \Pr(A = Z) = \sum_{i=1}^k \Pr(Z = i, A = i) = \sum_{i=1}^k \Pr(Z = i \mid A = i)\Pr(A = i)  = \sum_{i=1}^k q_{ii}p_i.
    \end{align*}
    Therefore, the utility constraint becomes:
    \begin{equation}\label{Q_constraint_utility}
        \sum_{i=1}^k q_{ii}p_i \geq 1 - \zeta.
    \end{equation}
    
\end{enumerate}
This utility constraint aims to regulate the trade-off between fairness and utility. Without the utility constraint, if we only consider the data fairness objective function, we may converge to a scenario where perfect data fairness is achieved at the cost of poor utility. This constraint prevents such a scenario. The parameter $\zeta$ allows us to control the utility of the LDP mechanism. 
% In binary case, since the entire LDP mechanism was defined using only two parameters $p$ and $q$, we could directly find the optimal equality constraint for utility purpose. However, in the non-binary sensitive attribute case, finding

Given (\ref{Q_constraints_truth_and_lie_2}), (\ref{Q_constraint_LDP}) can be reduced to:
\begin{equation}\label{Q_constraint_LDP_reduced}
        q_{jj} - e^\eps q_{ij} \leq 0 \hspace{2cm} \forall i,j \in [k].
    \end{equation}

Altogether, the optimization problem is as follows: 
\allowdisplaybreaks[0]
\begin{align}\label{eqn: original optimization_problem}
    & \hspace{-0.2cm}\min\limits_{\mathbf{Q}} \hspace{0.2cm} \Delta(D_M) \\ 
    & \text{s.t.} \hspace{1cm} q_{jj} - e^\eps q_{ij} \leq 0 \hspace{5.05cm} \forall i,j \in [k] \nonumber \\
    & \hspace{1.4cm} \sum\limits_{j=1}^k q_{ij} = 1 \hspace{5.65cm} \forall i \in [k] \nonumber \\
    & \hspace{1.5cm} q_{ij} \geq 0, \hspace{0.5cm} q_{ii} \geq q_{ij} \hspace{0.5cm} q_{jj} \geq q_{ij,} \hspace{2.42cm} \forall i,j \in [k] \nonumber \\
    % & \hspace{1.5cm} q_{ii} \geq q_{ij} \hspace{3.0cm} \forall i,j \in [k] \nonumber \\
    % & \hspace{1.5cm} q_{jj} \geq q_{ij} \hspace{3cm} \forall i,j \in [k] \nonumber \\
    & \hspace{1.4cm} \sum_{i=1}^k q_{ii}p_i \geq 1-\zeta \nonumber.
\end{align}
\allowdisplaybreaks

% \allowdisplaybreaks[0]
% \begin{align}\label{eqn: original optimization_problem}
%     & \hspace{-0.2cm}\min\limits_{\mathbf{Q}} \hspace{0.2cm} \Delta(D_M^\eps) \nonumber \\ \\
%     & \text{s.t.} \hspace{1cm} q_{jj} - e^\eps q_{ij} \leq 0 \hspace{2cm} \forall i,j \in [k] \nonumber \\
%     & \hspace{1.4cm} \sum\limits_{j=1}^k q_{ij} = 1 \hspace{2.65cm} \forall i \in [k] \nonumber \\
%     & \hspace{1.5cm} q_{ij} \geq 0 \hspace{3.2cm} \forall i,j \in [k] \nonumber \\
%     & \hspace{1.5cm} q_{ii} \geq q_{ij} \hspace{3.0cm} \forall i,j \in [k] \nonumber \\
%     & \hspace{1.5cm} q_{jj} \geq q_{ij} \hspace{3cm} \forall i,j \in [k] \nonumber \\
%     & \hspace{1.4cm} \sum_{i=1}^k q_{ii}p_i \geq 1-\zeta \nonumber 
% \end{align}
% \allowdisplaybreaks

Now, we need to write $\Delta(D_M)$ in terms of the parameters of the optimization problem. i.e., entries  of matrix $\mathbf{Q}$. It can be shown that:
\begin{align*}
    \Delta(D_{M}) := \max_{a \in [k]} \left| \frac{\Pr(Y = 1 \mid Z = a)}{\Pr(Y = 1)} - 1 \right| = \max\limits_{a \in [k]} \hspace{0.1cm} \Biggl| \dfrac{\sum\limits_{j=1}^k p_{1|j}p_jq_{ja}}{\sum\limits_{j=1}^k \Pr(Y=1) p_j q_{ja}} -   1 \Biggr|.
\end{align*}
To transform the equality constraint into an inequality constraint, we assume that 
$q_{ii} = 1 - \sum\limits_{\substack{j=1 \\ j \neq i}}^{k} q_{ij} $ for any $ i \in [k]$. With this assumption, optimization problem (\ref{eqn: original optimization_problem}) can be reformulated as:

% \begin{align}
%     \begin{cases}
%         & \min\limits_{\mathbf{Q}}  \hspace{0.2cm} \max\limits_{a \in [k]} \hspace{0.1cm} \Biggl| \dfrac{\sum\limits_{\substack{j = 1 \\ j \neq a}}^k p_{1|j}p_jq_{ja} + p_{1|a}p_a \Bigl(1 - \sum\limits_{\substack{j=1 \\ j \neq a}}^{k} q_{aj}\Bigr) - \sum\limits_{\substack{j=1 \\ j \neq a}}^k \Pr(Y=1) p_j q_{ja} - \Pr(Y=1) p_a \Bigl(1 - \sum\limits_{\substack{j=1 \\ j \neq a}}^{k} q_{aj}\Bigr)}{\sum\limits_{\substack{j=1 \\ j \neq a}}^k \Pr(Y=1) p_j q_{ja} + \Pr(Y=1) p_a \Bigl(1 - \sum\limits_{\substack{j=1 \\ j \neq a}}^{k} q_{aj}\Bigr)}  \Biggr|
%  \\ \\
%         & s.t. \hspace{1cm} \Bigl(1 - \sum\limits_{\substack{a=1 \\ a \neq j}}^{k} q_{ja}\Bigr) - e^\eps q_{ij} \leq 0 \hspace{2cm} \forall i,j \in [k], \quad i \neq j \\
%  & \hspace{1.5cm} \sum\limits_{\substack{a=1 \\ a \neq i}}^{k} q_{ia} \leq  1 \hspace{4.2cm}  \forall i \in [k]\\
%         & \hspace{1.5cm} q_{ij} \geq 0 \hspace{4.8cm} \forall i,j \in [k], \quad i \neq j \\
%         & \hspace{1.5cm} 1- \sum\limits_{\substack{a=1 \\ a \neq i}}^{k} q_{ia} \geq q_{ij} \hspace{3.4cm} \forall i,j \in [k], \quad i \neq j \\
%     & \hspace{1.5cm} 1- \sum\limits_{\substack{a=1 \\ a \neq j}}^{k} q_{ja} \geq q_{ij} \hspace{3.4cm} \forall i,j \in [k], \quad i \neq j \\
%     & \hspace{1.5cm} \sum\limits_{i=1}^k \Bigl( 1- \sum\limits_{\substack{a=1 \\ a \neq i}}^{k} q_{ia} \Bigr) p_i \geq 1-\zeta
%     \end{cases}
%     \label{eqn: minmax opt problem}
% \end{align}
\allowdisplaybreaks[0]
\begin{align} \label{eqn: minmax opt problem}
    & \min\limits_{\mathbf{Q}} \hspace{0.2cm} \max\limits_{a \in [k]} \hspace{0.1cm} \Biggl| \frac{\sum\limits_{\substack{j = 1 \\ j \neq a}}^k p_{1|j}p_jq_{ja} + p_{1|a}p_a q_{aa} - \sum\limits_{\substack{j=1 \\ j \neq a}}^k \Pr(Y=1) p_j q_{ja} - \Pr(Y=1) p_a q_{aa}}{\sum\limits_{\substack{j=1 \\ j \neq a}}^k \Pr(Y=1) p_j q_{ja} + \Pr(Y=1) p_a q_{aa}}  \Biggr| \\[-5pt] 
    & \text{s.t.} \hspace{1.2cm} \Bigl(1 - \sum\limits_{\substack{a=1 \\ a \neq j}}^{k} q_{ja}\Bigr) - e^\eps q_{ij} \leq 0 \hspace{5.96cm} \forall i,j \in [k], \quad i \neq j \nonumber \\[-5pt]
    & \hspace{1.8cm} \sum\limits_{\substack{a=1 \\ a \neq i}}^{k} q_{ia} \leq 1 \hspace{7.95cm} \forall i \in [k] \nonumber \\[-5pt]
    & \hspace{1.8cm} q_{ij} \geq 0, \hspace{0.5cm} 1 - \sum\limits_{\substack{a=1 \\ a \neq i}}^{k} q_{ia} \geq q_{ij}, \hspace{0.5cm}   1 - \sum\limits_{\substack{a=1 \\ a \neq j}}^{k} q_{ja} \geq q_{ij} \hspace{2.32cm} \forall i,j \in [k], \quad i \neq j \nonumber \\[-5pt]
    & \hspace{1.8cm} \sum\limits_{i=1}^k \Bigl(1 - \sum\limits_{\substack{a=1 \\ a \neq i}}^{k} q_{ia} \Bigr) p_i \geq 1 - \zeta \nonumber.
\end{align}
\allowdisplaybreaks
Note that the parameters of this optimization problem are the non-diagonal entries of the matrix \( \mathbf{Q} \), rather than all of its entries. The optimization problem~\eqref{eqn: minmax opt problem} can be cast as a min–max linear fractional program~\cite{jiao2022solving}, with an objective given by a ratio of two linear functions under linear inequality constraints. For fixed values of \( \epsilon \) and \( \zeta \), this problem can be solved numerically using the branch-and-bound method, as outlined in \cite{jiao2022solving}. Thus, for any given data distribution, we can numerically solve the problem to determine the optimal pre-processing mechanism with the desired \( \epsilon \) parameter for LDP and \( \zeta \) error parameter. We choose Definition \ref{def: data unfairness} over Definition \ref{def: data unfairness'} because the optimization problem (\ref{eqn: original optimization_problem}) can be expressed as a min-max linear fractional program when using the data unfairness metric \( \Delta \).

\subsection{Data Fairness Leads to Classification Fairness} \label{subsec: data fairness leads to classification fairness} 

In the previous sections, we derived optimal LDP mechanisms that minimize data unfairness. 
Our broader goal, however, is to ensure that training on the transformed data also leads to reduced classification unfairness.
Because classification unfairness depends on the specific classifier, 
it is difficult to make a universal claim. 
To establish a formal link, we identify a sufficient condition under which 
reducing data unfairness guarantees a reduction in classification unfairness. 
This connection, expressed through the concept of DA-optimal classifiers, 
is formalized in the following theorem.

\begin{theorem}\label{thm: data fairness leads to classification fairness}
Let \(P_{XAY}\) and \(Q_{XAY}\) be two joint distributions over 
\(\mathcal{X} \times \{0,1\} \times \{0,1\}\), and 
\(\mathcal{H}^*\) denote the set of classifiers satisfying 
\(\Pr(\hat{Y}=1)=\Pr(Y=1)\). 
Assume that the sensitive attribute has the same marginal under both distributions, 
i.e., \(\Pr_{(X,A,Y)\sim P_{XAY}}(A = a) = \Pr_{(X,A,Y)\sim Q_{XAY}}(A = a)\) for all \(a\in\{0,1\}\), 
and \(\Delta'(P_{XAY}) \le \Delta'(Q_{XAY})\). 
If \(h_P\) and \(h_Q\) are DA-optimal classifiers in \(\mathcal{H}^*\) 
with equal accuracy, trained on data from \(P_{XAY}\) and \(Q_{XAY}\) respectively, 
then
\[
\Delta_{\mathrm{SP}}(h_P) \le \Delta_{\mathrm{SP}}(h_Q).
\]

\end{theorem}

This theorem links data and classification unfairness through the DA-optimality condition. Under this condition, reducing data unfairness yields less discriminatory classifiers. In the definition of DA-optimality, we adopt the statistical parity gap~\(\Delta_{\mathrm{SP}}\) as the discrimination metric and use~\(\Delta_{\mathrm{SP}}(\hat{h})\) and~\(\Delta'(D)\) as the measures of classification and data unfairness, respectively, since both capture dependencies between labels and sensitive attributes in the same way.

Having theoretically examined the optimal LDP mechanisms for minimizing data unfairness, the next section will present experiments evaluating the performance of these mechanisms in classification tasks. These experiments will address both binary and non-binary sensitive attribute optimization problems.

% Experiments
\section{Experiments}\label{sec: experiments}
The optimal mechanisms introduced in Section~\ref{sec: technical_results} are designed to minimize data unfairness while satisfying utility and LDP constraints. We refer to both the binary and non-binary mechanisms as OPT. In this section, we empirically evaluate the effectiveness of OPT across different classification settings. Section~\ref{subsec: binary-nonbinary experiments} compares OPT with existing LDP mechanisms for both binary and multi-valued sensitive attributes. Section~\ref{subsec: mozannar exp} evaluates OPT against RR under the fairness framework of \cite{mozannar2020fair}. Finally, Section~\ref{subsec: fair projection experiments} compares OPT with Fair Projection~\cite{alghamdi2022beyond} and FairBalance~\cite{yu2024fairbalance}, two state-of-the-art post- and pre-processing fairness methods~\cite{wang2024aleatoric,soremekun2022software}.%
\footnote{
Experimental code is publicly available at \url{https://github.com/hradghoukasian/ldp_fairness}.}

\subsection{Comparison with Existing LDP Mechanisms}\label{subsec: binary-nonbinary experiments}

Following~\cite{arcolezi2025group}, we evaluate the proposed mechanism on the Adult~\cite{UCI} and Law School Admissions Council (LSAC)~\cite{wightman1998lsac} datasets, where the label \(Y\) indicates income above~\$26k and bar-exam pass, respectively. Depending on the setting (binary or multi-valued), the sensitive attribute is chosen among \texttt{race}, \texttt{gender}, \texttt{race--gender}, or \texttt{family income}.

We compare OPT with GRR (or RR in the binary case) and SS, which are known to provide strong privacy–utility–fairness trade-offs among existing LDP mechanisms~\cite{arcolezi2025group}.
 Each mechanism is evaluated across multiple privacy levels~\(\varepsilon\) and against a non-private baseline. The classifier~\(\hat{h}\) is a gradient-boosted tree (LightGBM)~\cite{ke2017lightgbm}, trained using an 80/20 train–test split and averaged over 20 random seeds. Predictions are made on the original (non-perturbed) test data. Utility is measured by accuracy, and fairness by the equalized-opportunity gap~\(\Delta_{\mathrm{EO}}(\hat{h})\) and statistical-parity gap~\(\Delta_{\mathrm{SP}}(\hat{h})\), with 95\% confidence intervals shown in the plots.

For binary attributes, we apply OPT, RR, and SS to perturbed training data, while the non-private baseline is trained on the original data. Here, OPT corresponds to the closed-form mechanism in Theorem~\ref{thm: binary optimal mechanism}, using \texttt{gender} as the sensitive attribute. 
Across privacy levels, OPT maintains accuracy comparable to RR, SS, and the non-private baseline, while consistently reducing unfairness (Figures~\ref{fig:adult2}--\ref{fig:lsac2}). As~\(\varepsilon\) increases, OPT yields larger fairness gains: on LSAC, both~\(\Delta_{\mathrm{SP}}(\hat{h})\) and~\(\Delta_{\mathrm{EO}}(\hat{h})\) are nearly halved; on Adult, fairness improves by at least~2\%. Notably, OPT does not converge to the non-private mechanism as~\(\varepsilon \to \infty\), since one sensitive value remains perturbed with probability~\(1/2\). Also, when~\(A\) is binary, SS and RR coincide.

\begin{figure}[htpb]
  \centering
  \includegraphics[width=0.7\columnwidth]{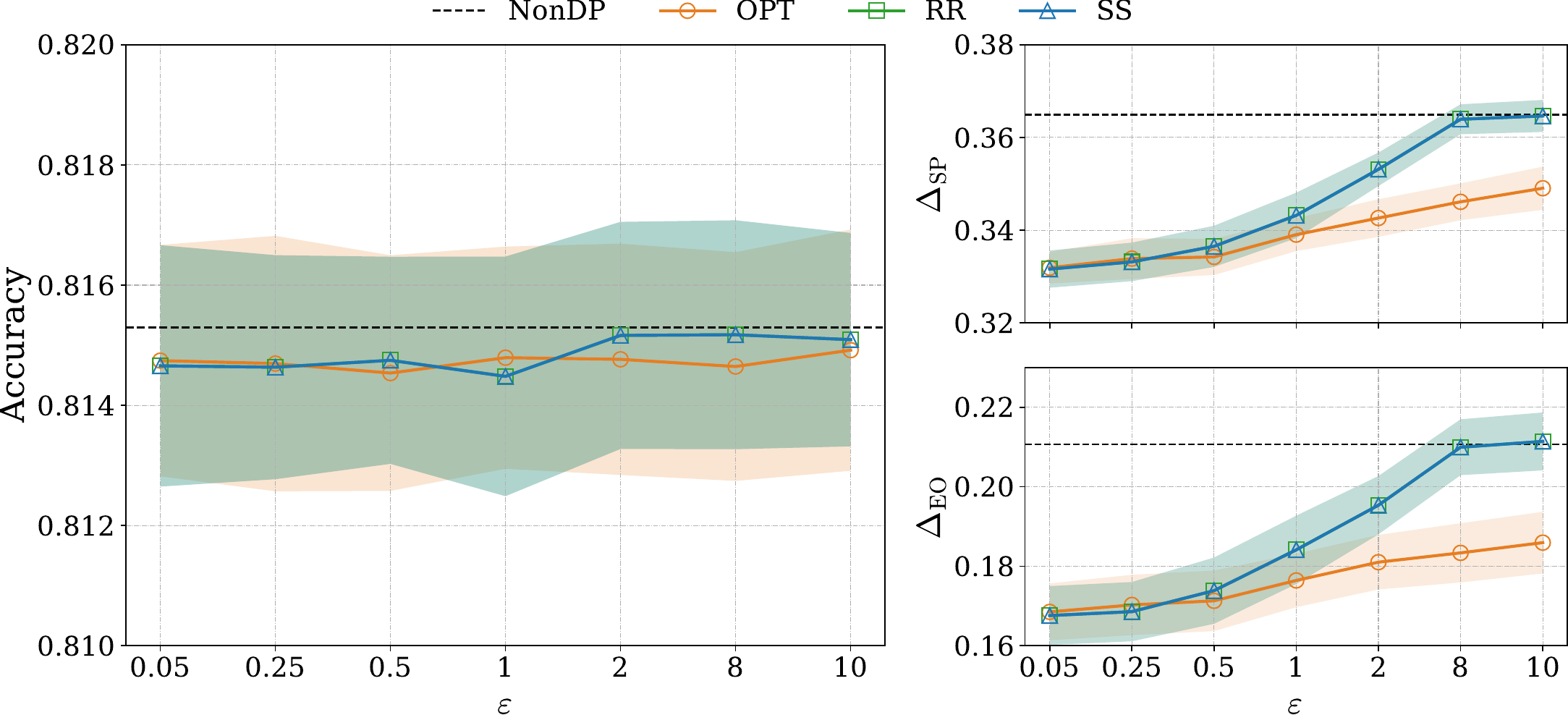}
  \caption{Adult dataset with binary sensitive attribute \texttt{gender}. 
  Left: accuracy; right: fairness metrics (statistical parity gap and equalized opportunity gap).}
  \label{fig:adult2}
\end{figure}

\begin{figure}[htpb]
  \centering
  \includegraphics[width=0.7\columnwidth]{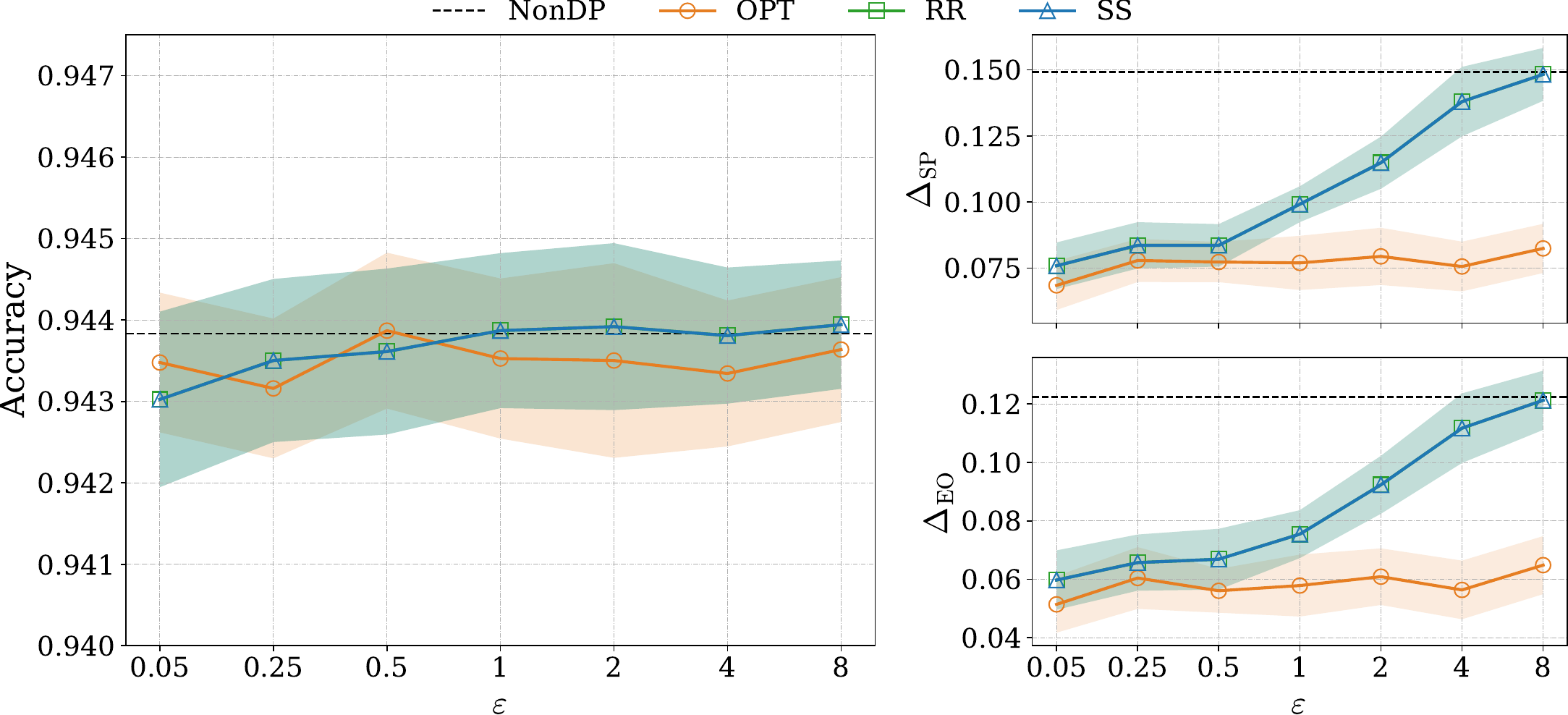}
  \caption{LSAC dataset with binary sensitive attribute \texttt{gender}. 
  Left: accuracy; right: fairness metrics (statistical parity gap and equalized opportunity gap).}
  \label{fig:lsac2}
\end{figure}

% Across privacy levels (Figures~\ref{fig:adult2}–\ref{fig:lsac2}), 
% OPT maintains accuracy comparable to GRR, SS, and the non-DP baseline, 
% while consistently reducing unfairness 
% (Figures~\ref{fig:adult2}–\ref{fig:lsac2}). 
% As \(\varepsilon\) increases, the fairness gap widens in favor of OPT: 
% on LSAC, both \(\Delta_{\mathrm{SP}}(\hat{h})\) and 
% \(\Delta_{\mathrm{EO}}(\hat{h})\) are nearly halved; 
% on Adult, OPT yields at least a 2\% reduction in both metrics. 
% Note that in the binary case, OPT does \emph{not} converge to the non-private 
% mechanism as \(\varepsilon\) grows, since one sensitive value remains perturbed 
% with probability \(1/2\). 
% Also, when the sensitive attribute is binary, SS and GRR coincide.

We next evaluate non-binary attributes: \texttt{race} (\(k=5\)) and \texttt{race--gender} (\(k=10\)) for Adult, and \texttt{family income} (\(k=5\)) for LSAC. 
The optimal mechanism is computed by solving the min--max linear fractional program in~\eqref{eqn: original optimization_problem} following~\cite{jiao2022solving}, using the smallest feasible~\(\zeta\) for each~\(\varepsilon\). 
Results are compared against GRR, SS, and the non-private baseline.

s

\begin{figure}[htpb]
  \centering
  \includegraphics[width=0.7\columnwidth]{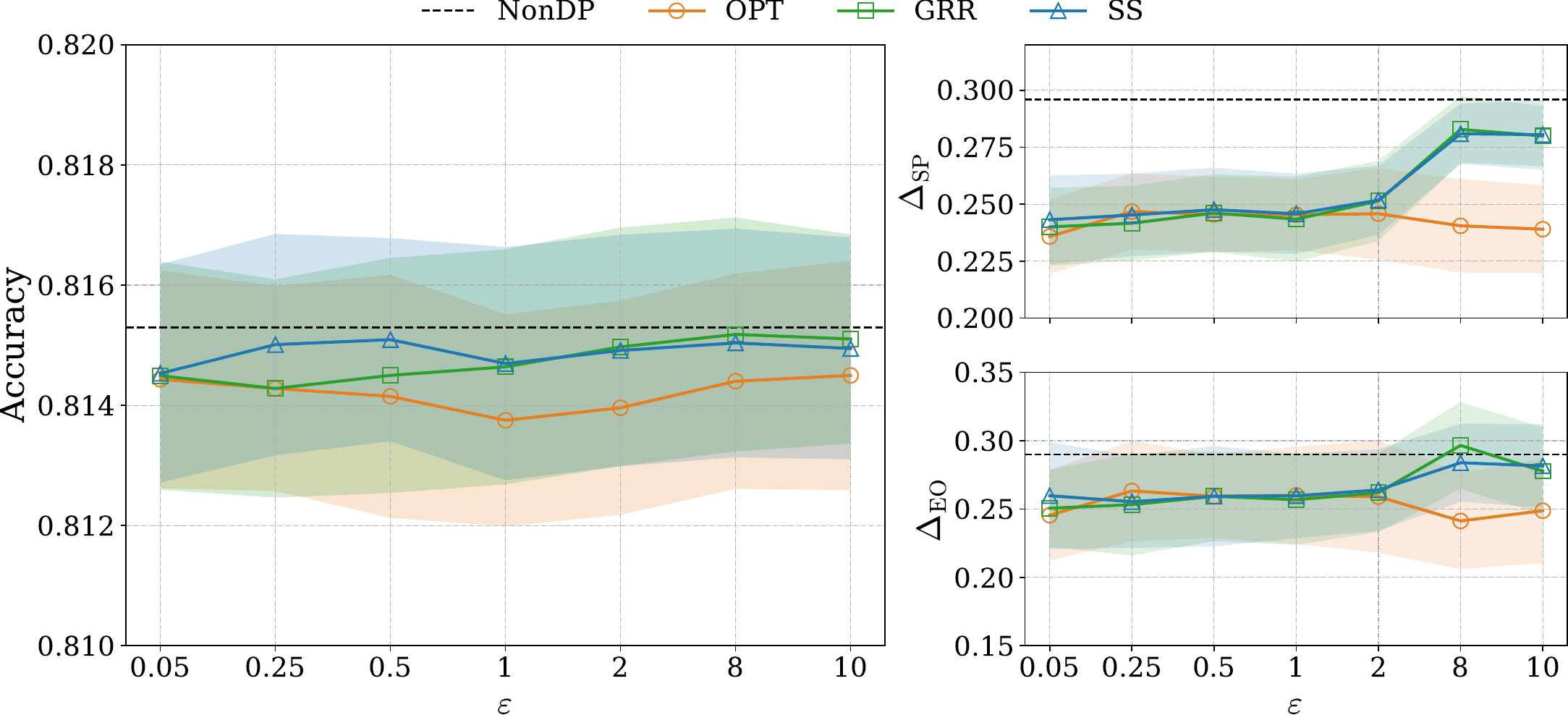}
  \caption{Adult dataset with 5-level sensitive attribute \texttt{race}. 
  Left: accuracy; right: fairness metrics (statistical parity gap and equalized opportunity gap).}
  \label{fig:adult5}
\end{figure}

\begin{figure}[htpb]
  \centering
  \includegraphics[width=0.7\columnwidth]{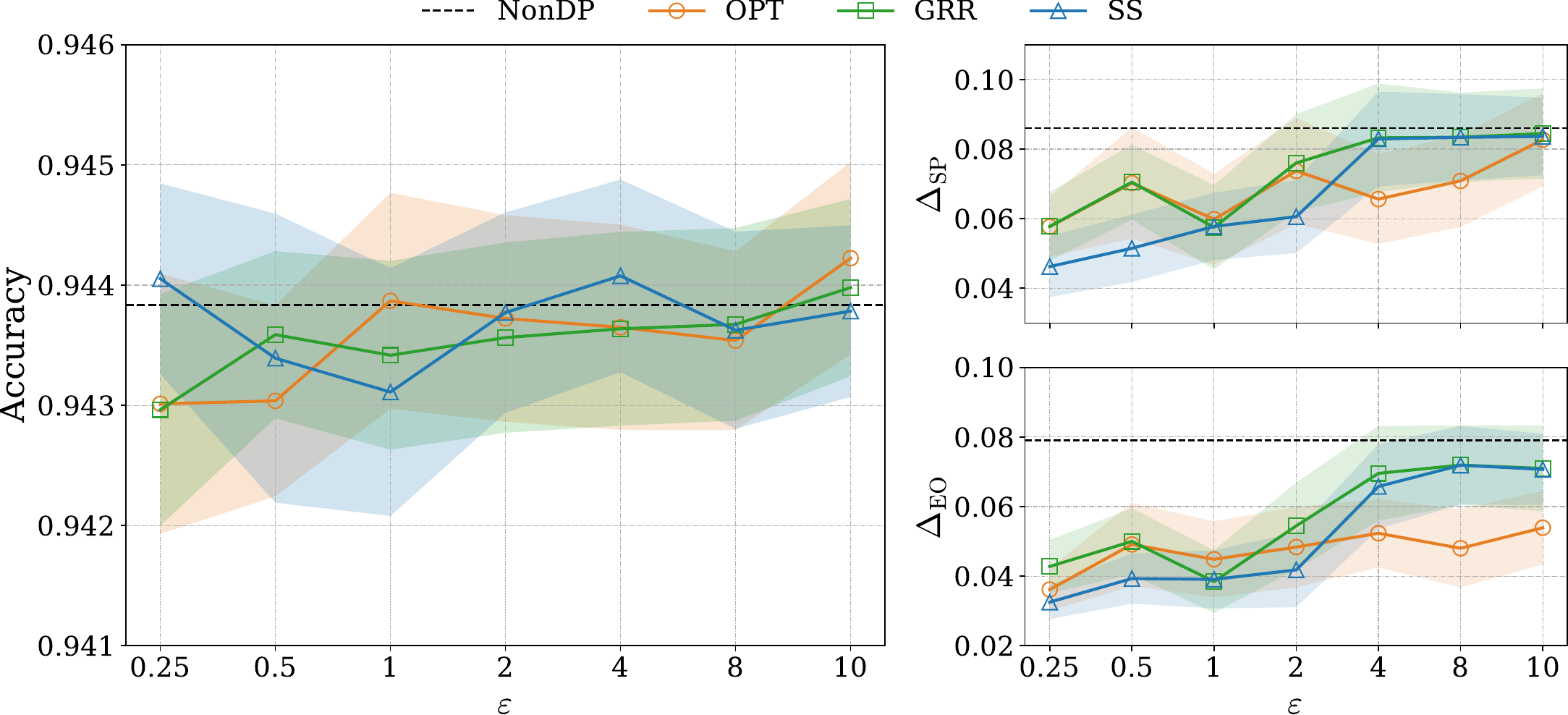}
  \caption{LSAC dataset with 5-level sensitive attribute \texttt{family-income}. 
  Left: accuracy; right: fairness (statistical parity gap and equalized opportunity gap).}
  \label{fig:lsac5}
\end{figure}
\begin{figure}[htpb]
  \centering
  \includegraphics[width=0.7\columnwidth]{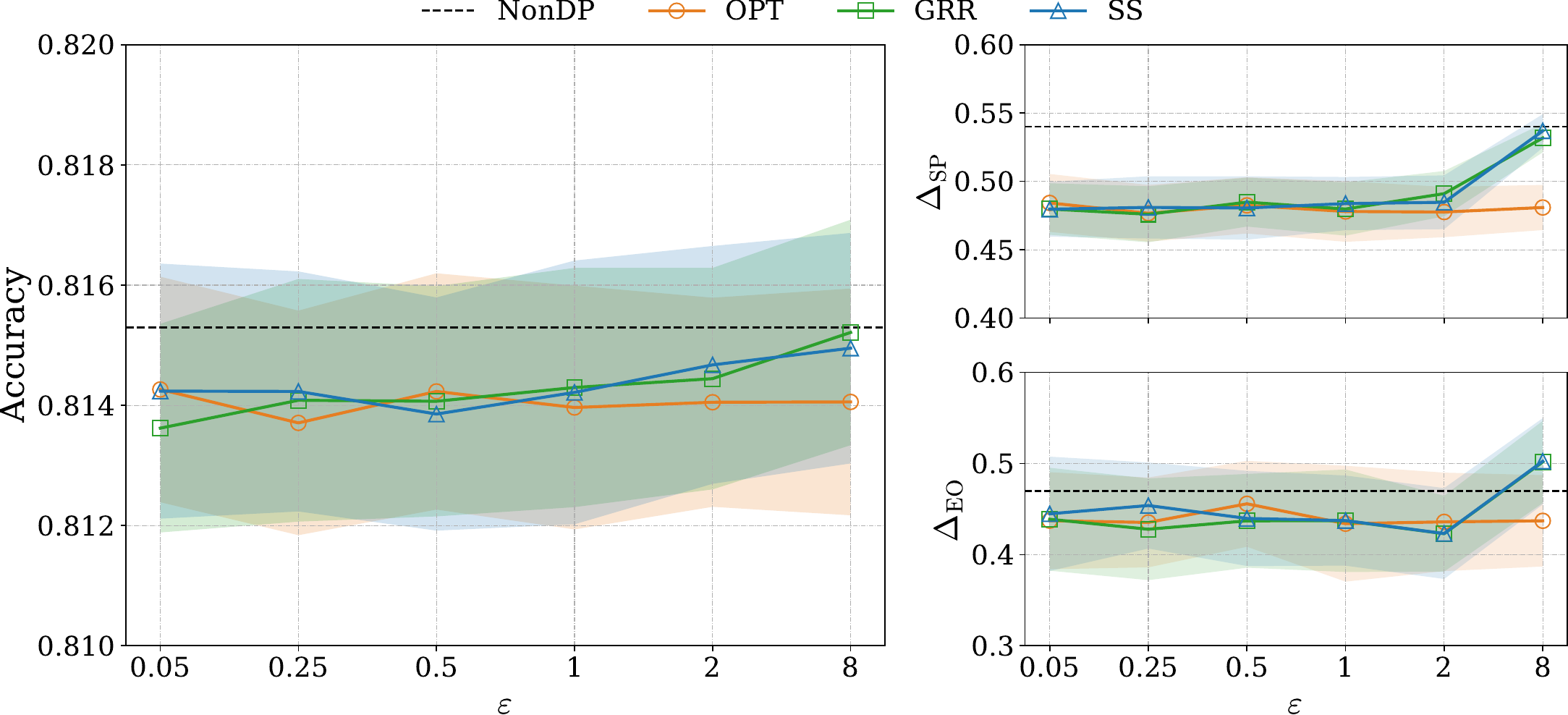}
  \caption{Adult dataset with 10-level sensitive attribute \texttt{race-gender}. 
  Left: accuracy; right: fairness (statistical parity gap and equalized opportunity gap).}
  \label{fig:adult10}
\end{figure}

As shown in Figures~\ref{fig:adult5}--\ref{fig:adult10}, OPT improves fairness while preserving accuracy. For small~\(\varepsilon\), its performance is comparable to GRR and SS in fairness metrics, whereas for larger~\(\varepsilon\), it achieves lower~\(\Delta_{\mathrm{SP}}(\hat{h})\) and~\(\Delta_{\mathrm{EO}}(\hat{h})\) with nearly unchanged accuracy, demonstrating improved fairness--utility trade-offs.
Although a closed-form solution is unavailable for non-binary~\(A\), the numerical OPT effectively reduces unfairness with minimal accuracy loss. As the support size~\(k\) increases, unfairness gaps grow larger, reflecting the greater heterogeneity introduced by a larger number of groups; nevertheless, OPT remains effective even for~\texttt{race--gender} (\(k=10\)), despite the increased computational cost.

\subsection{Comparison under the \cite{mozannar2020fair} Framework}\label{subsec: mozannar exp}

We integrate OPT into the two-step post-processing framework of~\cite{mozannar2020fair}, which is designed to train fair classifiers when only privatized sensitive attributes are available. In their original setup, RR is used to ensure LDP; we replace RR with OPT. The procedure first trains a base classifier using the privatized attributes, then post-processes its predictions to enforce equalized odds. To ensure a fair comparison and demonstrate the adaptability of OPT, we apply OPT in both the pre-processing of sensitive attributes and the post-processing step.

Following their setup, we train logistic regression models using 75\% of the data and test on the remaining 25\%. 
Results are averaged over 20 random splits on the Adult and LSAC datasets (sensitive attribute: \texttt{gender}) under varying~$\varepsilon$. 
Curves report mean performance with 95\% confidence intervals.

\begin{figure}[htpb]
  \centering
  \subfloat[Adult dataset]{%
    \includegraphics[width=0.35\columnwidth]{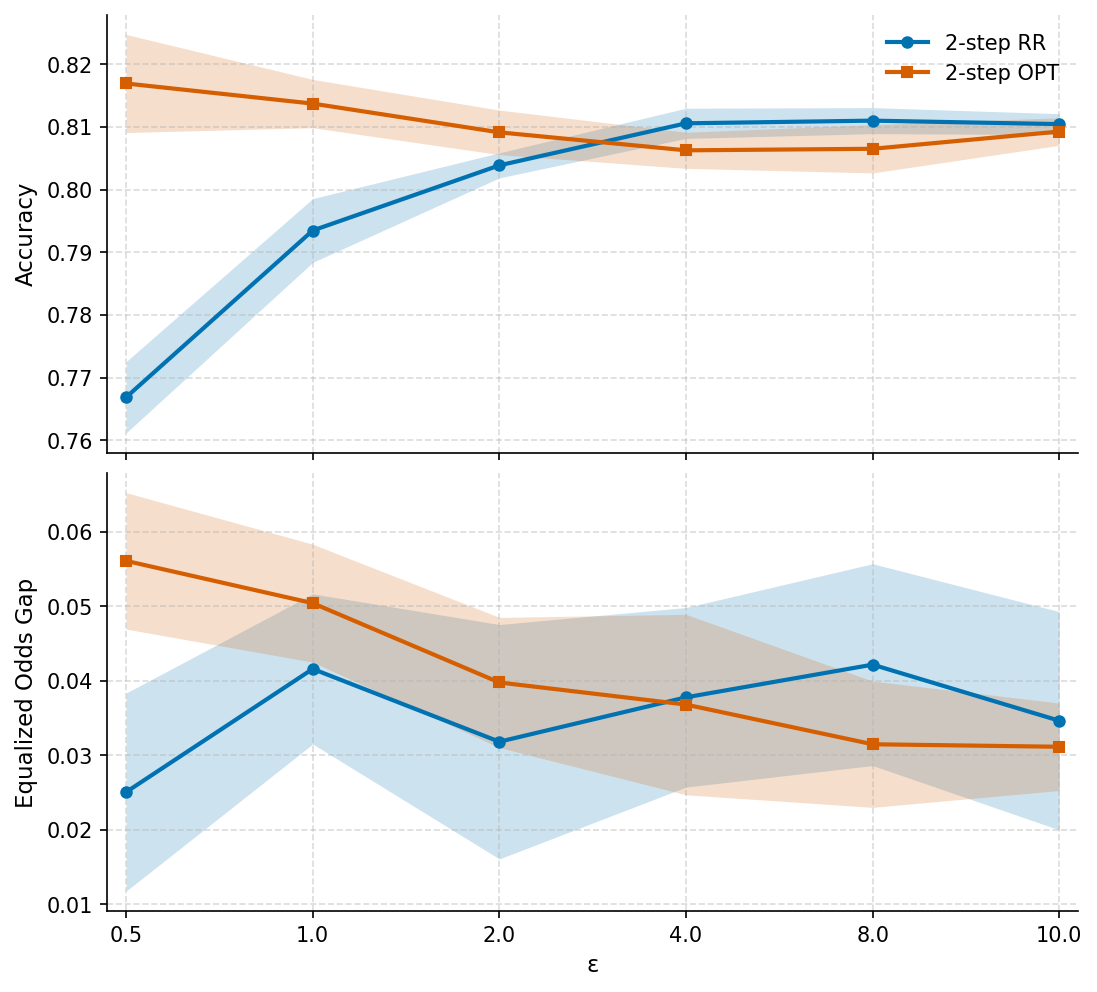}%
    \label{fig:mozannar_adult}}
  \hfil
  \subfloat[LSAC dataset]{%
    \includegraphics[width=0.35\columnwidth]{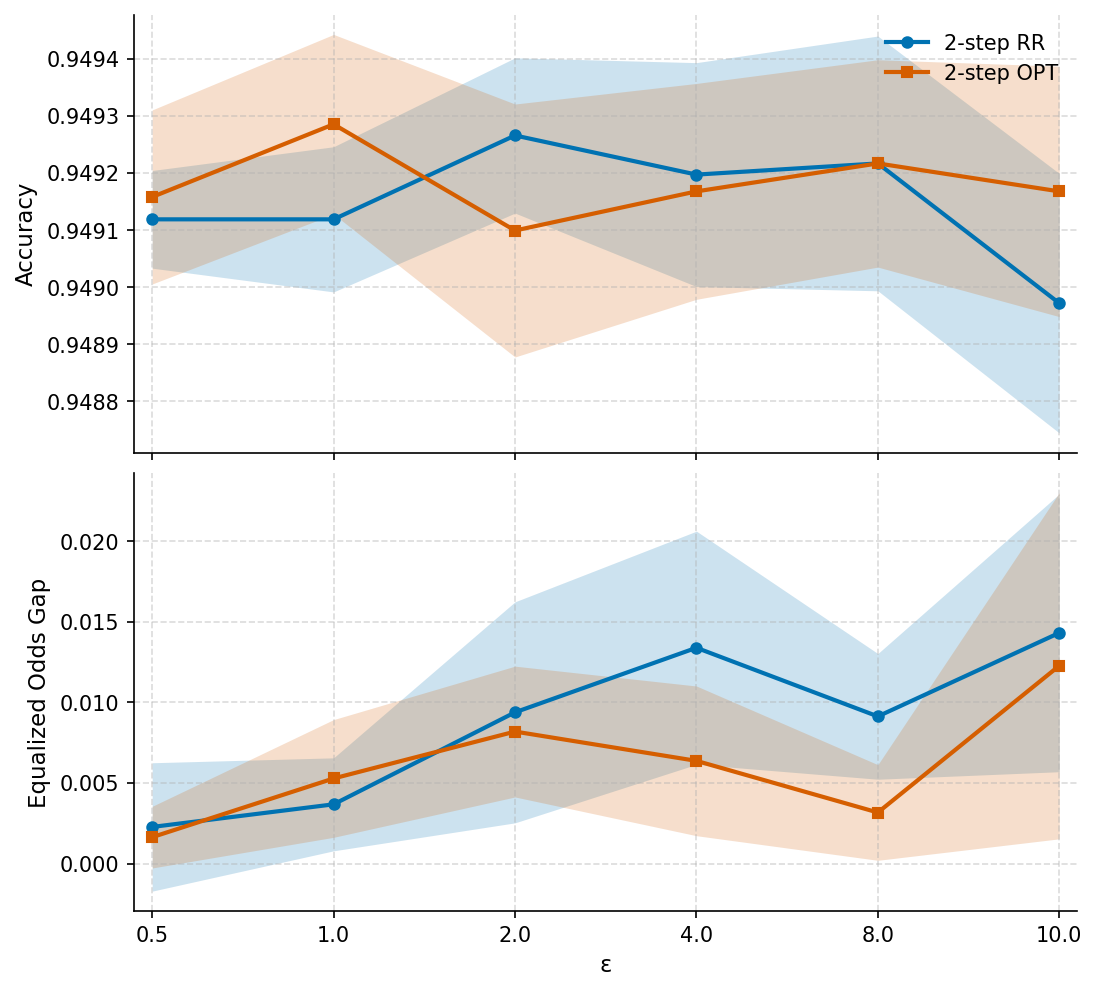}%
    \label{fig:mozannar_lsac}}
  \caption{Comparison of the two-step framework of \cite{mozannar2020fair} using RR and OPT. }
  \label{fig:mozannar_adult_lsac}
\end{figure}

Figure~\ref{fig:mozannar_adult_lsac} shows that on LSAC, OPT and RR achieve similar accuracy, while OPT yields smaller \(\Delta_{\mathrm{EOd}}\) across most~\(\varepsilon\) values. 
On Adult, OPT attains lower~\(\Delta_{\mathrm{EOd}}\) for larger~\(\varepsilon\) with comparable accuracy, and higher accuracy for smaller~\(\varepsilon\) at the cost of increased unfairness. These results indicate that OPT can effectively replace RR in the~\cite{mozannar2020fair} framework.

\subsection{Comparison with Fairness Pre/Post-Processing Methods}\label{subsec: fair projection experiments}

We further compare OPT with two state-of-the-art fairness interventions: 
FairProjection (post-processing) and 
FairBalance (pre-processing). 
% Both aim to reduce group-level disparities while preserving accuracy.

\textbf{FairProjection.} FairProjection is a post-processing fairness method that maps a pre-trained, potentially unfair classifier onto the space of models satisfying a chosen group-fairness criterion. The projected model is constructed by post-processing the predictions of the original classifier to enforce the target fairness metric.
Following~\cite{alghamdi2022beyond}, we compare OPT and FairProjection on the Adult and COMPAS~\cite{angwin2022machine} datasets. 
In Adult, the label \(Y\) and the sensitive attribute are defined as before; 
in COMPAS, \(Y\) denotes recidivism with sensitive attribute~\texttt{race}. 
Both methods are evaluated under mean equalized odds and statistical parity constraints using gradient boosting. 
Models use a 70/30 train–test split, averaged over 30 trials with 95\% confidence intervals. 
FairProjection provides a controllable fairness--accuracy trade-off by allowing one to specify a target unfairness gap to be satisfied. To enable a fair comparison with OPT, we evaluate OPT at~$\varepsilon = 1$ and compare it to FairProjection at matched accuracy, reporting the corresponding~$\Delta_{\mathrm{SP}}$ and~$\Delta_{\mathrm{MEO}}$ values (Table~\ref{tab:opt_fp_results}).

\textbf{FairBalance.}
FairBalance is a pre-processing bias mitigation algorithm that balances the class distribution within each demographic group by assigning appropriate weights to training samples. Its goal is to identify the underlying causes of fairness violations and mitigate them without altering the standard training procedure.
Following~\cite{yu2024fairbalance}, we evaluate OPT and FairBalance on the Bank Marketing, Student Performance (Math), and German Credit datasets~\cite{UCI}.
In Bank, \(Y\) indicates term-deposit subscription and the sensitive attribute is~\texttt{age}; 
in Math, \(Y\) denotes final grade above a threshold and the sensitive attribute is~\texttt{gender}; 
in German, \(Y\) indicates credit risk and the sensitive attribute is~\texttt{gender}. 
All models use logistic regression with 70/30 train--test splits, averaged over 30 trials and 95\% confidence intervals. 
We fix \(\varepsilon=1\) and compare OPT to FairBalance, reporting~\(\Delta_{\mathrm{EO}}\) and~\(\Delta_{\mathrm{MEO}}\) (Table~\ref{tab:opt_fb_results}).

\begin{table}[htpb]
  \centering
  \renewcommand{\arraystretch}{1.1}
  \caption{FairProjection vs.\ OPT on multiple datasets.
Accuracy, $\Delta_{\mathrm{SP}}$, and $\Delta_{\mathrm{MEO}}$ reported with 95\% Confidence intervals (percent).}
  \label{tab:opt_fp_results}
  \begin{tabular}{|
    >{\centering\arraybackslash}m{0.25\columnwidth}|
    >{\centering\arraybackslash}m{0.16\columnwidth}|
    >{\centering\arraybackslash}m{0.16\columnwidth}|
    >{\centering\arraybackslash}m{0.16\columnwidth}|}
    \hline
    Dataset–Method & Accuracy & $\Delta_{\mathrm{SP}}$ & $\Delta_{\mathrm{MEO}}$ \\
    \hline
    Adult–FairProjection & $86.1 \pm 0.1\%$ &  $11.5 \pm 0.3\%$ & $6.7 \pm 0.8\%$\\
    \hline
    Adult–OPT ($\eps = 1$) & $86.1 \pm 0.1\%$ & $11.0 \pm 0.3\%$ & $5.8 \pm 0.6\%$ \\
    \hline
    COMPAS–FairProjection & $68.1 \pm 0.3 \%$ & $26.8\pm 0.8 \%$ & $ 22.3\pm 0.8\%$ \\
    \hline
    COMPAS–OPT ($\eps = 1$) & $68.1 \pm 0.3 \%$ & $24.3\pm 1.1 \%$ & $ 20.2\pm 0.9\%$  \\
    \hline
  \end{tabular}
\end{table}
\begin{table}[htpb]
  \centering
  \renewcommand{\arraystretch}{1.1}
  \caption{FairBalance vs.\ OPT on multiple datasets.
Accuracy, $\Delta_{\mathrm{EO}}$, and $\Delta_{\mathrm{MEO}}$ reported with 95\% Confidence intervals (percent).}
  \label{tab:opt_fb_results}
  \begin{tabular}{|
    >{\centering\arraybackslash}m{0.25\columnwidth}|
    >{\centering\arraybackslash}m{0.16\columnwidth}|
    >{\centering\arraybackslash}m{0.16\columnwidth}|
    >{\centering\arraybackslash}m{0.16\columnwidth}|}
    \hline
    Dataset–Method & Accuracy & $\Delta_{\mathrm{EO}}$ & $\Delta_{\mathrm{MEO}}$ \\
    \hline
    Bank–FairBalance        & $84.5 \pm 0.1\%$  & $8.6 \pm 1.4\%$  & $2.3 \pm 0.7\%$ \\
    \hline
    Bank–OPT ($\eps = 1$)   & $90.1 \pm 0.1\%$  & $3.7 \pm 0.9\%$  & $2.7 \pm 0.6\%$ \\
    \hline
    % Heart–FairBalance       & $82.8 \pm 1.4\%$  & $11.0 \pm 3.5\%$ & $9.5 \pm 2.2\%$ \\
    % \hline
    % Heart–OPT ($\eps = 1$)  & $83.1 \pm 1.4\%$  & $8.6 \pm 1.9\%$  & $7.8 \pm 1.6\%$ \\
    % \hline
    Math–FairBalance        & $91.7 \pm 0.9\%$  & $5.8 \pm 1.7\%$  & $5.2 \pm 1.3\%$ \\
    \hline
    Math–OPT ($\eps = 1$)   & $91.6 \pm 0.9\%$  & $5.1 \pm 1.5\%$  & $5.1 \pm 1.2\%$ \\
    \hline
    German–FairBalance      & $70.4 \pm 0.8\%$  & $19.7 \pm 3.6\%$ & $17.2 \pm 3.1\%$ \\
    \hline
    German–OPT ($\eps = 1$) & $76.3 \pm 0.9\%$  & $19.7 \pm 3.1\%$ & $20.6 \pm 2.8\%$ \\
    \hline
  \end{tabular}
\end{table}
% As shown in Table~\ref{tab:opt_fp_results}, OPT ($\varepsilon = 1$) achieves up to a 2.5\% improvement in fairness at a fixed accuracy level. Table~\ref{tab:opt_fb_results} further shows that across all datasets, OPT ($\varepsilon = 1$) consistently attains higher accuracy than FairBalance while substantially reducing $\Delta_{\mathrm{EO}}$ or keeping it unchanged. However, this improvement may come at the cost of a larger $\Delta_{\mathrm{MEO}}$. 
As shown in Table~\ref{tab:opt_fp_results}, OPT evaluated at~$\varepsilon = 1$ achieves up to a 2.5\% improvement in fairness at a fixed accuracy level. 
On the Adult dataset, OPT attains comparable accuracy (86.1\%) to FairProjection while reducing $\Delta_{\mathrm{SP}}$ from 11.5\% to 11.0\% and $\Delta_{\mathrm{MEO}}$ from 6.7\% to 5.8\%. 
A similar trend is observed on COMPAS, where OPT achieves a smaller $\Delta_{\mathrm{SP}}$ (24.3\% vs.\ 26.8\%) and $\Delta_{\mathrm{MEO}}$ (20.2\% vs.\ 22.3\%) at matched accuracy (68.1\%). 

Table~\ref{tab:opt_fb_results} further shows that across all datasets, OPT ($\varepsilon = 1$) consistently attains higher accuracy than FairBalance---for instance, 90.1\% vs.\ 84.5\% on Bank and 76.3\% vs.\ 70.4\% on German---while substantially reducing $\Delta_{\mathrm{EO}}$ or keeping it comparable. 
However, this improvement may come at the cost of a slightly larger $\Delta_{\mathrm{MEO}}$.

\section{Conclusion}\label{sec: conclusion}

Motivated by recent empirical findings on the impact of local differential privacy LDP on fairness in classification, this work introduces an \textit{optimal LDP mechanism} applied to sensitive attributes to minimize data unfairness. 
For binary~$A$, we derive a closed-form expression for the optimal mechanism; for multi-valued~$A$, we provide an equivalent formulation solvable numerically. 
We further show that when DA–optimal classifiers in~$\mathcal{H}^*$ are trained on data with reduced unfairness, the resulting classifiers also exhibit improved fairness (measure by statistical parity). 
These theoretical results are supported by experiments on multiple real-world datasets. 
On the Adult and LSAC datasets, the proposed mechanism effectively reduces classification unfairness while maintaining utility comparable to state-of-the-art LDP mechanisms at a fixed~$\varepsilon$-LDP level. 
Moreover, OPT can directly replace RR mechanism in the fairness framework of~\cite{mozannar2020fair}, enabling fairness-aware learning when only the privatized sensitive attribute is available. 
Finally, our approach achieves a better accuracy–fairness trade-off than both pre- and post-processing fairness interventions, while preserving the privacy of sensitive attributes across the Adult, COMPAS, Bank, Math, and German datasets.

In conclusion, finding the optimal LDP mechanism and demonstrating its effectiveness in reducing classification unfairness reinforces the view that, while central DP can be detrimental to fairness, LDP offers a beneficial alternative. The randomization introduced by LDP has been both theoretically proven and empirically validated as effective in reducing classification unfairness. This paper positions LDP as a simple yet powerful tool for mitigating data unfairness, emphasizing its potential as an approach for enhancing fairness in machine learning.

% References
\bibliographystyle{unsrt}  
\bibliography{references}  

% Appendix
\section{Appendix}\label{sec: appendix}

\begin{lemma}\label{lemma: data metrics are equivalent}
   \( \Delta(D) \leq c_1 \Delta'(D) \) and \( \Delta'(D) \leq c_2 \Delta(D) \) for constants \( c_1 \) and \( c_2 \) dependent on the marginal distribution $P_Y$ of the joint distribution \( P_{XAY} \).
\end{lemma}

\begin{proof}
For the first part, we have:
    \begin{align*}
    \Delta(D) & = \max_{a \in [k]} \left| \frac{\Pr(Y = 1 \mid A = a)}{\Pr(Y = 1)} - 1 \right| \\
    & = \max_{a \in [k]} \left| \frac{\Pr(Y = 1 \mid A = a) - \Pr(Y = 1)}{\Pr(Y = 1)}  \right| \\
    & \leq \max_{a,a' \in [k]} \left| \frac{\Pr(Y = 1 \mid A = a) - \Pr(Y = 1 \mid A = a')}{\Pr(Y = 1)}  \right| \\
    & = \frac{1}{\Pr(Y = 1)}\Delta'(D), 
\end{align*}
where the third line follows from the fact that $\min\limits_{a' \in [k]}  \Pr(Y = 1 \mid A = a') \leq \Pr(Y = 1 )$.

 Additionally, for the second part, it can be shown that:

\begin{align*}
    \Delta'(D) & = \max_{a,a' \in [k]} \Bigl| \Pr(Y = 1 \mid A = a) - \Pr(Y = 1 \mid A = a')\Bigr| \\
    & = \max_{a,a' \in [k]} \Bigl| \Pr(Y = 1 \mid A = a) - \Pr(Y = 1) + \Pr(Y = 1) - \Pr(Y = 1 \mid A = a')\Bigr| \\
    & \leq \max_{a,a' \in [k]} \Biggl[ \Bigl| \Pr(Y = 1 \mid A = a) - \Pr(Y = 1)\Bigr| + \Bigl| \Pr(Y = 1 \mid A = a') - \Pr(Y = 1) \Bigr| \Biggr]  \\
    & \leq 2 \max_{a \in [k]} \Bigl| \Pr(Y = 1 \mid A = a) - \Pr(Y = 1)\Bigr| \\
    & = 2  \Pr(Y = 1) \Delta(D).
\end{align*}

It follows that $c_1 = \frac{1}{\Pr(Y = 1)}$ and $c_2 = 2  \Pr(Y = 1)$, where $c_1$ and $c_2$ only depend on the marginal distribution $P_Y$ of the joint distribution \( P_{XAY} \).
\end{proof}

\begin{proof}[Proof of Lemma~\ref{lemma: RR makes data fairer}]

We know: \begin{equation*}
    \Delta'(D) := \max_{a,a' \in [k]} \Bigl| \Pr(Y = 1 \mid A = a) - \Pr(Y = 1 \mid A = a')\Bigr| = \max_{a \in [k]}  p_{1|a} -  \min_{a \in [k]}p_{1|a}.
\end{equation*}

\textbf{Proof Sketch:} W.L.O.G., we assume $\max_{a \in [k]}  p_{1|a} = p_{1|k}$ and  $\min_{a \in [k]}  p_{1|a} = p_{1|1}$. Therefore, we have $\Delta'(D) = p_{1|k} - p_{1|1}$. Similarly, we can express $\Delta'(D^{\varepsilon}_{GRR})$ as:
\begin{align*}
    \Delta'(D^{\varepsilon}_{GRR}) = \max_{a \in [k]}  \Pr(Y=1 \,|\, Z = a) -  \min_{a \in [k]}\Pr(Y=1 \,|\, Z = a).
\end{align*}
For each $a \in [k]$, we prove $p_{1|1} \leq \Pr(Y=1 \,|\, Z = a) \leq p_{1|k}$. This means that $\max_{a \in [k]}  \Pr(Y=1 \,|\, Z = a) \leq p_{1|k}$ and $p_{1|1} \leq \min_{a \in [k]}\Pr(Y=1 \,|\, Z = a)  $, concluding ~$\Delta'(D^{\varepsilon}_{GRR})~\leq~\Delta'(D)$.

$\forall a \in [k]$, We have:
\begin{align*}
        \Pr(Y=1 \,|\, Z = a)   & =  \frac{\Pr(Y=1 , Z = a)}{\Pr(Z = a)} \\
        & = \frac{ \sum\limits_{j = 1}^k \Pr(Y=1 , Z = a, A = j)}{\sum\limits_{j = 1}^k \Pr(Z = a,A = j)} \\
        & = \frac{\sum\limits_{j = 1}^k \Pr(Y=1 , A = j)\Pr(Z = a \,|\, A = j, Y=1) }{\sum\limits_{j = 1}^k \Pr(Z = a \,|\, A = j) \Pr(A = j)} \\
        & = \frac{ \sum\limits_{j = 1}^k \Pr(Y=1 , A = j)\Pr(Z = a \,|\, A = j) }{\sum\limits_{j = 1}^k \Pr(Z = a \,|\, A = j) \Pr(A = j)} \\
        & = \frac{\sum\limits_{j = 1}^k \Pr(Y=1 \,|\, A = j)\Pr(A = j)\Pr(Z = a \,|\, A = j) }{\sum\limits_{j = 1}^k \Pr(Z = a \,|\, A = j) \Pr(A = j) }  \\
        & = \frac{p_{1|a} p_a \pi + \sum\limits_{\substack{j \in [k] \\ j \neq a}} p_{1|j} p_j \bar{\pi} }{p_a \pi + \sum\limits_{\substack{j \in [k] \\ j \neq a}} p_j \bar{\pi}}.
\end{align*}
We know $p_{1|1} \leq p_{1|j} \leq p_{1|k}$ for all values of $j \in [k]$. Therefore, it can be obtained that:

\begin{equation*}
    \frac{p_{1|1} p_a \pi + \sum\limits_{\substack{j \in [k] \\ j \neq a}} p_{1|1} p_j \bar{\pi} }{p_a \pi + \sum\limits_{\substack{j \in [k] \\ j \neq a}} p_j \bar{\pi}} \leq 
    \frac{p_{1|a} p_a \pi + \sum\limits_{\substack{j \in [k] \\ j \neq a}} p_{1|j} p_j \bar{\pi} }{p_a \pi + \sum\limits_{\substack{j \in [k] \\ j \neq a}} p_j \bar{\pi}} \leq 
    \frac{p_{1|k} p_a \pi + \sum\limits_{\substack{j \in [k] \\ j \neq a}} p_{1|k} p_j \bar{\pi} }{p_a \pi + \sum\limits_{\substack{j \in [k] \\ j \neq a}} p_j \bar{\pi}}.
\end{equation*}
Thus, we have:
\begin{equation*}
   p_{1|1}\leq 
    \frac{p_{1|a} p_a \pi + \sum\limits_{\substack{j \in [k] \\ j \neq a}} p_{1|j} p_j \bar{\pi} }{p_a \pi + \sum\limits_{\substack{j \in [k] \\ j \neq a}} p_j \bar{\pi}} \leq 
    p_{1|k} \hspace{0.5cm } \longrightarrow  \hspace{0.5cm}p_{1|1}\leq 
     \Pr(Y=1 \,|\, Z = a) \leq 
    p_{1|k} \hspace{0.5cm} \forall a \in [k].
\end{equation*}
It follows that:
\begin{align*}
     \Delta'(D^{\varepsilon}_{GRR}) & = \max_{a \in [k]}  \Pr(Y=1 \,|\, Z = a) -  \min_{a \in [k]}\Pr(Y=1 \,|\, Z = a) \\
     & \leq p_{1|k} - p_{1|1} \\
     & = \Delta'(D).
\end{align*}
Now the proof is complete. 
\end{proof}

\begin{proof}[Proof of Theorem \ref{thm: binary optimal mechanism}]
We know: \begin{equation*}
        \Delta'(D) = \max_{a,a' \in \{0,1\}} \Bigl| \Pr(Y = 1 \mid A = a) - \Pr(Y = 1 \mid A = a')\Bigr|. 
\end{equation*}
Similarly, we can express $\Delta'(D_{M})$ as:
\begin{align*}
    \Delta'(D_{M}) = \max_{a,a' \in \{0,1\}} \Bigl| \Pr(Y = 1 \mid Z = a) - \Pr(Y = 1 \mid Z = a')\Bigr|.
\end{align*}
\textbf{Proof Sketch:} We assume that $\max_{a \in \{0,1\}}  p_{1|a} = p_{1|1}$ and  $\min_{a \in \{0,1\}}  p_{1|a} = p_{1|0}$. Therefore, we have $\Delta'(D) = p_{1|1} - p_{1|0}$. 
Proof consists of three steps. In the first step, we find the set of feasible values of $p$ and $q$ such that the mechanism $M$ satisfies $\eps$-LDP. In the second step, we prove that $\Pr(Y=1 \,|\, Z = 0) \leq \Pr(Y=1 \,|\, Z = 1)$. Finally, in the third step, we determine the values of $p$ and $q$ that minimize $  
  \frac{\Delta'(D_{M})}{\Delta'(D)}$ over all mechanisms $M$ such that $\varepsilon^\star(M) = \varepsilon$.
% , thereby concluding the proof of $\min\limits_{\substack{\eps_0-\text{LDP} M \\  \eps_0 \geq \eps}} \frac{\Delta'(D^{\varepsilon_0}_{M})}{\Delta'(D)} \neq  \frac{\Delta'(D^{\varepsilon}_{GRR})}{\Delta'(D)}$.

\textbf{Step 1:} 
We start by assuming that \( \frac{1}{2} \leq p \leq 1 \) and \( \frac{1}{2} \leq q \leq 1 \). This is important because simply replacing the sensitive attribute of each individual randomly with \( 0 \) or \( 1 \) (with a misreporting probability of \( \frac{1}{2} \)) would satisfy \( 0 \)-LDP, but it would severely compromise utility. Therefore, to ensure non-trivial utility, it is necessary that \( \frac{1}{2} \leq p \leq 1 \) and \( \frac{1}{2} \leq q \leq 1 \). Also, by definition of an $\eps$-LDP mechanism, we have:
% \begin{equation*}
%     E_{e^\eps}(\text{Ber}(1-p)\, || \, \text{Ber}(q)) = E_{e^\eps}(\text{Ber}(q)\, || \, \text{Ber}(1-p)) = 0
% \end{equation*}
% It follows that:
\begin{align*}
    (1-p) - e^\eps q & \leq 0 \\ 
   p - e^\eps(1-q) & \leq 0 \\
   q - e^\eps (1-p) & \leq 0 \\
    (1-q) - e^\eps p & \leq 0.
\end{align*}
From $\frac{1}{2} \leq p \leq  1$  and  $\frac{1}{2} \leq q \leq  1$ directly follows that $(1-p) - e^\eps q \leq 0 $ and $ (1-q) - e^\eps p \leq  0$. However, to satisfy $p - e^\eps(1-q) \leq 0  $ and $ q - e^\eps (1-p) \leq  0$ we should have:
\begin{equation*}
    p - e^\eps(1-q) \leq 0 \implies p \leq e^\eps(1-q).
\end{equation*}
And
\begin{equation*}
    q - e^\eps (1-p) \leq 0 \implies q \leq e^\eps(1-p).
\end{equation*}
In conclusion, any mechanism $M$ that guarantees $\eps$-LDP and non-trivial utility should satisfy the following inequalities:
\begin{equation*}
    \frac{1}{2} \leq p \leq  1, \hspace{0.5cm}
 \frac{1}{2} \leq q \leq  1, \hspace{0.5cm} p \leq e^\eps(1-q), \hspace{0.5cm} q \leq e^\eps(1-p).
\end{equation*}

\textbf{Step 2:}  $\forall a \in \{0,1\}$, We have:
\begin{align*}
        \Pr(Y=1 \,|\, Z = a)   & =  \frac{\Pr(Y=1 , Z = a)}{\Pr(Z = a)} \\
        & = \frac{ \sum\limits_{j = 0}^1 \Pr(Y=1 , Z = a, A = j)}{\sum\limits_{j = 0}^1 \Pr(Z = a,A = j)} \\
        & = \frac{\sum\limits_{j = 0}^1 \Pr(Y=1 , A = j)\Pr(Z = a \,|\, A = j, Y=1) }{\sum\limits_{j = 0}^1 \Pr(Z = a \,|\, A = j) \Pr(A = j)} \\
        & = \frac{ \sum\limits_{j = 0}^1 \Pr(Y=1 , A = j)\Pr(Z = a \,|\, A = j) }{\sum\limits_{j = 0}^1 \Pr(Z = a \,|\, A = j) \Pr(A = j)} \\
        & = \frac{\sum\limits_{j = 0}^1 \Pr(Y=1 \,|\, A = j)\Pr(A = j)\Pr(Z = a \,|\, A = j) }{\sum\limits_{j = 0}^1 \Pr(Z = a \,|\, A = j) \Pr(A = j) } .
\end{align*}

Therefore, we have:
\begin{equation*}
     \Pr(Y=1 \,|\, Z = 0) = \frac{p_{1|0} p_0 p + p_{1|1}p_1(1-q)}{p_0 p + p_1 (1-q)},
\end{equation*}
And
\begin{equation*}
     \Pr(Y=1 \,|\, Z = 1) = \frac{p_{1|0} p_0 (1-p) + p_{1|1}p_1(q)}{p_0 (1-p) + p_1 q}.
\end{equation*}

To demonstrate that $\Pr(Y=1 \,|\, Z = 0) \leq \Pr(Y=1 \,|\, Z = 1) $, we need to prove the following inequality:
\begin{equation*}
\biggl(p_0 (1-p) + p_1 q\biggr)\biggl( p_{1|0} p_0 p + p_{1|1}p_1(1-q)\biggr) \leq \biggl(p_0 p + p_1 (1-q)\biggr)\biggl( p_{1|0} p_0 (1-p) + p_{1|1}p_1(q)\biggr).
\end{equation*}

Expanding both sides of the inequality, we have:
\begin{equation*}
\begin{aligned}
 & p_{1|0} (p_0)^2 p(1-p) + p_{1|1} p_0 p_1 (1-p) (1-q) 
 + p_{1|0}  p_0 p_1 p q   + p_{1|1} (p_1)^2 q (1-q) \leq \\
 & \hspace{2.7cm}p_{1|0} (p_0)^2 p(1-p) + p_{1|1} p_0 p_1 p q 
 + p_{1|0} p_0 p_1 (1-q)   (1-p) + p_{1|1} (p_1)^2 q (1-q).
\end{aligned}
\end{equation*}

Simplifying this, we get:
\begin{equation*}
p_{1|1}p_0p_1(1-p)(1-q) + p_{1|0}p_0p_1pq \leq p_{1|1}p_0p_1pq + p_{1|0}p_0p_1(1-p)(1-q).
\end{equation*}
This simplifies further to:

\begin{equation*}
p_{1|1}(1-p)(1-q) + p_{1|0}pq \leq p_{1|1}pq + p_{1|0}(1-p)(1-q).
\end{equation*}

Rearranging, we obtain:
\begin{align*}
& p_{1|0}pq - p_{1|0}(1-p)(1-q) \leq p_{1|1}pq - p_{1|1}(1-p)(1-q)  \\ & \hspace{3cm} \iff
 p_{1|0} \bigl( pq - (1-p)(1-q)\bigr) \leq p_{1|1} \bigl( pq - (1-p)(1-q)\bigr).
\end{align*}

This inequality holds true because $p_{1|0} \leq p_{1|1}$ and $\bigl( pq - (1-p)(1-q)\bigr) \geq 0$. Therefore, we can conclude that $\Pr(Y=1 \,|\, Z = 0) \leq \Pr(Y=1 \,|\, Z = 1) $. 

\textbf{Step 3:}  From steps 1 and 2, we can conclude that the problem we are interested in becomes the following optimization problem.

\begin{align}
\label{eq:opt_problem}
        & \min\limits_{p,q} \dfrac{\dfrac{p_{1|0} p_0 (1-p) + p_{1|1}p_1(q)}{p_0 (1-p) + p_1 q} - \dfrac{p_{1|0} p_0 p + p_{1|1}p_1(1-q)}{p_0 p + p_1 (1-q)}}{p_{1|1} - p_{1|0}} \nonumber \\
        & s.t. \hspace{1cm} \frac{1}{2} \leq p \leq  1, \hspace{0.5cm}
 \frac{1}{2} \leq q \leq  1, \\
 & \hspace{1cm} \hspace{0.6cm} p \leq e^\eps(1-q), \hspace{0.5cm} q \leq e^\eps(1-p) \nonumber.
\end{align}

Now, we simplify the objective function of the optimization problem above.

\begin{align}\label{equiivalent_optimization_problem_def1}
    \hspace{2cm} & \hspace{-2cm} \dfrac{\dfrac{p_{1|0} p_0 (1-p) + p_{1|1}p_1(q)}{p_0 (1-p) + p_1 q} - \dfrac{p_{1|0} p_0 p + p_{1|1}p_1(1-q)}{p_0 p + p_1 (1-q)}}{p_{1|1} - p_{1|0}} \nonumber\\
    &  = \dfrac{p_{1|0}p_0p_1(1-p)(1-q) + p_{1|1}p_0p_1pq - p_{1|0}p_0p_1pq - p_{1|1}p_0p_1(1-p)(1-q)}{\bigl( p_0 (1-p) + p_1 q \bigr) \bigl( p_0 p + p_1 (1-q)\bigr)\bigl(p_{1|1} - p_{1|0}\bigr)}  \nonumber\\
    & = \dfrac{\bigl(p_0p_1\bigr)\bigl(p_{1|0}(1-p)(1-q) + p_{1|1}pq - p_{1|0}pq - p_{1|1}(1-p)(1-q)\bigr)}{\bigl( p_0 (1-p) + p_1 q \bigr) \bigl( p_0 p + p_1 (1-q)\bigr)\bigl(p_{1|1} - p_{1|0}\bigr)}  \nonumber\\
    & = \dfrac{\bigl(p_0p_1\bigr)\bigl(p_{1|1}-p_{1|0}\bigr)\bigl(pq-(1-p)(1-q)\bigr)}{\bigl( p_0 (1-p) + p_1 q \bigr) \bigl( p_0 p + p_1 (1-q)\bigr)\bigl(p_{1|1} - p_{1|0}\bigr)} \nonumber\\
    & = \dfrac{\bigl(p_0p_1\bigr)\bigl(pq-(1-p)(1-q)\bigr)}{\bigl( p_0 (1-p) + p_1 q \bigr) \bigl( p_0 p + p_1 (1-q)\bigr)}.
\end{align}

Since the optimization is over the parameters $p$ and $q$, \eqref{eq:opt_problem} is equivalent to the following optimization problem:
\begin{align}
\label{eq:opt_problem2}
        & \min\limits_{p,q} \dfrac{pq-(1-p)(1-q)}{\bigl( p_0 (1-p) + p_1 q \bigr) \bigl( p_0 p + p_1 (1-q)\bigr)} \nonumber \\
        & s.t. \hspace{1cm} \frac{1}{2} \leq p \leq  1, \hspace{0.5cm}
 \frac{1}{2} \leq q \leq  1,  \\
 & \hspace{1cm} \hspace{0.6cm} p \leq e^\eps(1-q), \hspace{0.5cm} q \leq e^\eps(1-p) \nonumber.
\end{align}

Suppose~$\varepsilon$ is given and we aim to design the optimal $\varepsilon$-LDP mechanism. 
Let~$p$ and~$q$ denote the probabilities of truthfully reporting the sensitive attribute of an individual when the original attribute is~$0$ and~$1$, respectively. 
For a fixed~$\varepsilon$, higher values of~$p$ and~$q$ yield improved utility. 
Let~$p < e^\varepsilon(1 - q)$ and~$q < e^\varepsilon(1 - p)$. 
In this case, for the given~$\varepsilon$, we can jointly increase~$p$ and~$q$ until one of the inequalities becomes tight; that is, we have either~$p = e^\varepsilon(1 - q)$ and~$q \le e^\varepsilon(1 - p)$, or~$p \le e^\varepsilon(1 - q)$ and~$q = e^\varepsilon(1 - p)$. 
Therefore, we are only interested in those pairs~$(p,q)$ that satisfy one of these two constraints. 
Note that these constraints guarantee non-trivial utility. 
Specifically, for a fixed~$\varepsilon$, we consider the optimal achievable utility—corresponding to the minimum possible probability of misrepresenting the sensitive attribute. 
This aligns with our initial goal of optimizing the objective function 
\[
\min\limits_{M :\, \varepsilon^\star(M) = \varepsilon} 
  \frac{\Delta'(D_{M})}{\Delta'(D)}.
\]
These boundary conditions fully characterize the optimal trade-off between privacy and utility, as any strictly interior point~$(p,q)$ can be improved without violating the $\varepsilon$-LDP constraints.

We consider two cases. in case 1 we have: $\bigl(p = e^\eps(1-q) \hspace{0.1cm}\text{and} \hspace{0.1cm} q \leq e^\eps(1-p) \bigr)$ and in case 2 we have $\bigl( q = e^\eps(1-p) \hspace{0.1cm}\text{and} \hspace{0.1cm} p \leq e^\eps(1-q)\bigr)$.

\textbf{Case 1:} Here we have $p = e^\eps(1-q)$. The objective function becomes:

\begin{align*}
        & \min\limits_{q} \dfrac{e^\eps(1-q)q-(1 - e^\eps(1-q))(1-q)}{\bigl( p_0 (1 - e^\eps(1-q)) + p_1 q \bigr) \bigl( p_0 e^\eps(1-q) + p_1 (1-q)\bigr)} \\
        & s.t. \hspace{1cm} \frac{1}{2} \leq e^\eps(1-q) \leq 1, \hspace{0.5cm}
 \frac{1}{2} \leq q \leq  1, \\
 & \hspace{1.6cm}  q \leq e^\eps (1 - e^\eps(1-q)). 
\end{align*}

Rearranging the constraints, we can rewrite the objective function of case 1 as:
\begin{align}
\label{eqn:case1}
        & \min\limits_{q} \dfrac{e^\eps(1-q)q-(1 - e^\eps(1-q))(1-q)}{\bigl( p_0 (1 - e^\eps(1-q)) + p_1 q \bigr) \bigl( p_0 e^\eps(1-q) + p_1 (1-q)\bigr)} \nonumber\\
        & s.t. \hspace{1cm} 1 - e^{-\eps} \leq q \leq 1 - \frac{e^{-\eps}}{2}, \hspace{0.5cm}
 \frac{1}{2} \leq q \leq  1, \\
 & \hspace{1.6cm}  q \geq \frac{e^\eps}{e^\eps +1}  \nonumber .
\end{align}

Since $1-e^{-\eps} < \frac{e^\eps}{e^\eps + 1}$, we can further simplify the objective function as:

\begin{align*}
        & \min\limits_{q} \dfrac{e^\eps(1-q)q-(1 - e^\eps(1-q))(1-q)}{\bigl( p_0 (1 - e^\eps(1-q)) + p_1 q \bigr) \bigl( p_0 e^\eps(1-q) + p_1 (1-q)\bigr)} \\
        & s.t. \hspace{1cm} \frac{e^\eps}{e^\eps + 1} \leq q \leq 1 - \frac{e^{-\eps}}{2}, \hspace{0.5cm}
 \frac{1}{2} \leq q \leq  1.
\end{align*}
Now, we focus on the objective function of the problem \eqref{eqn:case1}. 

\begin{align*}
    & \dfrac{e^\eps(1-q)q-(1 - e^\eps(1-q))(1-q)}{\bigl( p_0 (1 - e^\eps(1-q)) + p_1 q \bigr) \bigl( p_0 e^\eps(1-q) + p_1 (1-q)\bigr)} \\
    & \hspace{2cm} = \dfrac{\bigl(1-q\bigr) \bigl( e^\eps q-(1 - e^\eps(1-q))\bigr)}{\bigl( p_0 (1 - e^\eps(1-q)) + p_1 q \bigr) \bigl( p_0 e^\eps + p_1 \bigr)\bigl(1-q\bigr)}  \\
    & \hspace{2cm} = \dfrac{ \bigl( e^\eps q-(1 - e^\eps(1-q))\bigr)}{\bigl( p_0 (1 - e^\eps(1-q)) + p_1 q \bigr) \bigl( p_0 e^\eps + p_1 \bigr)} \\
    & \hspace{2cm} = \dfrac{ \bigl( e^\eps -1 \bigr)}{\bigl( p_0 (1 - e^\eps(1-q)) + p_1 q \bigr) \bigl( p_0 e^\eps + p_1 \bigr)} .
\end{align*}

Given the constraints that $ \frac{e^\epsilon}{e^\epsilon +1} \leq q \leq 1 - \frac{e^{-\epsilon}}{2}$ and $\frac{1}{2} \leq q \leq  1$, since the optimization is over the variable $q$, we have:

\begin{align*}
    \argmin\limits_{q} \hspace{0.2cm} \dfrac{ \bigl( e^\epsilon -1 \bigr)}{\bigl( p_0 (1 - e^\epsilon(1-q)) + p_1 q \bigr) \bigl( p_0 e^\epsilon + p_1 \bigr)} & =  \argmin\limits_{q} \hspace{0.2cm}\dfrac{ 1}{\bigl( p_0 (1 - e^\epsilon(1-q)) + p_1 q \bigr) } \\
    & = \argmax\limits_{q} \hspace{0.2cm} \bigl( p_0 (1 - e^\epsilon(1-q)) + p_1 q \bigr) \\
    & = \argmax\limits_{q} \hspace{0.2cm} \bigl( p_0 - p_0 e^{\epsilon} + q (p_0 e^{\epsilon} + p_1) \bigr) \\
    & = \argmax\limits_{q} \hspace{0.2cm} q.
\end{align*}

Given the constraints $ \frac{e^\epsilon}{e^\epsilon +1} \leq q \leq 1 - \frac{e^{-\epsilon}}{2}$ and $\frac{1}{2} \leq q \leq  1$, the optimal value of $q$ will be $q^* = 1 - \frac{e^{-\epsilon}}{2}$. Therefore, in Case 1, the optimal solution is at $(p,q) = \left(\frac{1}{2}, 1 - \frac{e^{-\epsilon}}{2}\right)$. Similarly, in Case 2, the optimal value occurs at $(p,q) = \left(1 - \frac{e^{-\epsilon}}{2}, \frac{1}{2}\right)$. Hence, the minimum value arises from either of these cases. Now we show when each of the cases is optimal. In order for Case 1 to be the optimal value we should have:

\begin{equation*}
      \dfrac{\frac{1}{2}(1-\frac{e^{-\eps}}{2})-\frac{1}{2}(\frac{e^{-\eps}}{2})}{\bigl( p_0 (\frac{1}{2}) + p_1 (1-\frac{e^{-\eps}}{2}) \bigr) \bigl( p_0 (\frac{1}{2}) + p_1 (\frac{e^{-\eps}}{2})\bigr)} \leq \dfrac{(1-\frac{e^{-\eps}}{2})(\frac{1}{2})-(\frac{e^{-\eps}}{2})(\frac{1}{2})}{\bigl( p_0 (\frac{e^{-\eps}}{2}) + p_1 (\frac{1}{2}) \bigr) \bigl( p_0 (1-\frac{e^{-\eps}}{2}) + p_1 (\frac{1}{2})\bigr)}.
\end{equation*}

By rearranging, the inequality becomes:
 \begin{equation*}
     {\Bigl( p_0 (\frac{e^{-\eps}}{2}) + p_1 (\frac{1}{2}) \Bigr) \Bigl( p_0 (1-\frac{e^{-\eps}}{2}) + p_1 (\frac{1}{2})\Bigr)} \leq {\Bigl( p_0 (\frac{1}{2}) + p_1 (1-\frac{e^{-\eps}}{2}) \Bigr) \Bigl( p_0 (\frac{1}{2}) + p_1 (\frac{e^{-\eps}}{2})\Bigr)}.
 \end{equation*}

Therefore, we must have:
 \begin{equation*}
     p_0^2 (\frac{e^{-\eps}}{2})(1 - \frac{e^{-\eps}}{2}) + p_1^2(\frac{1}{4}) \leq  p_1^2 (\frac{e^{-\eps}}{2})(1 - \frac{e^{-\eps}}{2}) + p_0^2(\frac{1}{4}).
 \end{equation*}
Or, equivalently:
\begin{equation*}
    p_1^2 \leq p_0^2 \implies p_1 \leq p_0.
\end{equation*}
 
Specifically, when $p_0 < p_1$, the optimal $(p,q)$ is $\left(1 - \frac{e^{-\epsilon}}{2}, \frac{1}{2}\right)$, and when $p_1 < p_0$, it is $(p,q) = \left(\frac{1}{2}, 1 - \frac{e^{-\epsilon}}{2}\right)$. If $p_0 = p_1$, then both $\left(1 - \frac{e^{-\epsilon}}{2}, \frac{1}{2}\right)$ and  $\left(1 - \frac{e^{-\epsilon}}{2}, \frac{1}{2}\right)$ will be optimal solutions. We know that in randomized response $p = q = \frac{e^\eps}{e^\eps + 1}$ which is clearly not the optimal solution.

\end{proof} \vspace{0.5cm}

\begin{proof}[Proof of Theorem \ref{thm: data fairness leads to classification fairness}]

Let us recall the definitions from the main text that are used in this proof. The definition of $\Delta'(D)$ and $ \Delta_{\mathrm{SP}}(\hat{h})$ are as follows:
\begin{equation*}
    \Delta'(D) = \max_{a,a' \in \{0,1\}} \left| \Pr(Y = 1 \mid A = a) - \Pr(Y = 1 \mid A = a')\right|.
\end{equation*}
\begin{equation*}
    \Delta_{\mathrm{SP}}(\hat{h}) = \max_{a,a' \in \{0,1\}} \left| \Pr(\hat{Y} = 1 \mid A = a) - \Pr(\hat{Y} = 1 \mid A = a') \right|.
\end{equation*}
 
We know that $\Delta'(P_{XAY}) \leq \Delta'(Q_{XAY}) $ and we want to prove that $\Delta_{\mathrm{SP}}(h_P) \leq \Delta_{\mathrm{SP}}(h_Q)$.

In order to prove this theorem, we refer to Theorem 1 in \cite{kamiran2012data}.

Theorem 1: A classifier $h$ is DA-optimal in $\mathcal{H}^*$ iff
$$
\operatorname{acc}\left(h^{\text {Perf }}\right)-\operatorname{acc}(h)=\frac{2 n_0 n_1}{(n_0 + n_1)^2}\left(\Delta_{\mathrm{SP}}\left(h^{\text {Perf }}\right)-\Delta_{\mathrm{SP}}(h)\right),
$$
where  \( \mathcal{H}^* \) denotes the class of all classifiers satisfying \( \Pr(\hat{Y} = 1) = \Pr(Y = 1) \). $h^{\text {Perf }}$ refers to a perfect classifier, and $n_0$ and $n_1$ denote number of data points with $A = 0$ and $A=1$ on a dataset used for training respectively.  Note that as mentioned in Section \ref{subsec: data fairness leads to classification fairness}, statistical parity gap is used as a discrimination metric in the DA-optimality definition.

Since $\Pr_{(X,A,Y)\sim P_{XAY} } (A=a) = \Pr_{(X,A,Y)\sim Q_{XAY} } (A=a)$ for $a \in \{0,1\}$, it follows that $\frac{2 n_0 n_1}{(n_0 + n_1)^2}$ is a fixed value for two distributions $P_{XAY}$ and $Q_{XAY}$. We denote this constant by $c$.
In addition, since $h_P$ and $h_Q$ are DA-optimal classifiers in $\mathcal{H}^*$, we have:
\begin{align*}
    & 1 - \operatorname{acc}(h_P) = c \left(\Delta'\left(P_{XAY}\right)-\Delta_{\mathrm{SP}}(h_P)\right) \\
     & 1 - \operatorname{acc}(h_Q) = c \left(\Delta'\left(Q_{XAY}\right)-\Delta_{\mathrm{SP}}(h_Q)\right).
\end{align*}

Note that we used $\Delta'\left(P_{XAY}\right)$ and  $\Delta'\left(Q_{XAY}\right)$ rather than $\Delta_{\mathrm{SP}}\left(h^{\text {Perf }}\right)$ since the perfect classifier identically represents the unfairness of the data distribution.

Since we have assumed that the accuracy of learned classifiers $h_P$ and $h_Q$ are equal, we have:
\begin{align*}
    \left(\Delta'\left(P_{XAY}\right)-\Delta_{\mathrm{SP}}(h_P)\right) = \left(\Delta'\left(Q_{XAY}\right)-\Delta_{\mathrm{SP}}(h_Q)\right).
\end{align*}

From this equation, we can conclude that if $\Delta'\left(P_{XAY} \right)\leq \Delta'\left(Q_{XAY}\right)$, then $\Delta_{\mathrm{SP}}(h_P) \leq \Delta_{\mathrm{SP}}(h_Q)$.
    
\end{proof}

% \begin{abstract}
     
% \end{abstract}

% \section{Introduction}

% \section{Notation and Basic Definitions}\label{notation_definition}

% \section{Main Results}\label{technical_results}

% \subsection{Fair Post-Processing Without Privacy}\label{fair_post-proc_no_privacy}

% \subsection{Fair Post-Processing With Privacy}\label{fair_post-proc_with_privacy}

% \section{Experiments}\label{experiments}

% \cite{agarwal_2020}

% \bibliographystyle{unsrt}  
% \bibliography{references}  

% \appendix 

% \section*{Appendix A} \label{proofs}

% \section*{Appendix B}\label{appendix_experiments}

\end{document}